\documentclass{article}

\PassOptionsToPackage{numbers, compress}{natbib}

\usepackage[final]{neurips_2022}

\usepackage[utf8]{inputenc} 
\usepackage[T1]{fontenc}    
\usepackage{url}            
\usepackage{booktabs}       
\usepackage{amsfonts}       
\usepackage{nicefrac}       
\usepackage{microtype}      
\usepackage{xcolor}         
\usepackage{wrapfig}

\usepackage[utf8]{inputenc} 
\usepackage[T1]{fontenc}    
\usepackage[colorlinks,citecolor=green]{hyperref}
\usepackage{algorithm}
\usepackage{algorithmic}
\usepackage[ruled,linesnumbered,algo2e]{algorithm2e}
\usepackage{subfig}

\usepackage{subfig}
\usepackage{graphicx}
\usepackage{amsmath,url,amssymb}
\usepackage{amsfonts}
\usepackage{amsthm}
\usepackage{ctable}

\usepackage{multirow}
\usepackage{pifont}
\usepackage{color}
\usepackage{enumitem}
\usepackage{sidecap}
\usepackage{xcolor}
\usepackage[normalem]{ulem}
\newcommand{\dgal}{\textit{dynamicAL}}
\newcommand{\rom}[1]{\uppercase\expandafter{\romannumeral #1\relax}}

\newcommand{\mylistbegin}{
  \begin{list}{$\bullet$}
   {
     \setlength{\itemsep}{-2pt}
     \setlength{\leftmargin}{1em}
     \setlength{\labelwidth}{1em}
     \setlength{\labelsep}{0.5em} } }
\newcommand{\mylistend}{
   \end{list}  }

\newtheorem{prop}{Proposition}
\newtheorem{rem}{Remark}
\newtheorem{thm}{Theorem}
\newtheorem{lem}{Lemma}

\newtheorem{defi}{Definition}
\def\norm#1{\Vert#1\Vert}

\newcommand{\ie}{\emph{i.e.}} 

\DeclareMathOperator{\Tr}{Tr}

\newcommand{\mathbbm}[1]{\text{\usefont{U}{bbm}{m}{n}#1}}

\makeatletter
\newcommand{\nosemic}{\renewcommand{\@endalgocfline}{\relax}}

\title{Deep Active Learning by Leveraging Training Dynamics}

\usepackage{authblk}
\usepackage{xpatch}
\xpatchcmd{\author}{\relax#1\relax}{\relax\detokenize{#1}\relax}{}{}
\author[\empty]{\textbf{Haonan Wang}\textsuperscript{$1$}}
\author[\empty]{\textbf{Wei Huang}\textsuperscript{$2$}}
\author[\empty]{\textbf{Ziwei Wu}\textsuperscript{$1$}}
\author[\empty]{\textbf{Andrew Margenot}\textsuperscript{$1$}}
\author[\empty]{\\\textbf{Hanghang Tong}\textsuperscript{$1$}}
\author[\empty]{\textbf{Jingrui He}\textsuperscript{$1$}}
\affil[$1$]{University of Illinois Urbana-Champaign}
\affil[$2$]{University of New South Wales}
\affil[$1$]{\textit {\{haonan3,ziweiwu2,margenot,htong,jingrui\}@illinois.edu}}
\affil[$2$]{\textit {\{weihuang.uts\}@gmail.com}}

\begin{document}

\maketitle

\begin{abstract}
    \label{abs}
    Active learning theories and methods have been extensively studied in classical statistical learning settings. However, deep active learning, i.e., active learning with deep learning models, is usually based on empirical criteria without solid theoretical justification, thus suffering from heavy doubts when some of those fail to provide benefits in real applications.
    In this paper, by exploring the connection between the generalization performance and the training dynamics, we propose a theory-driven deep active learning method (\textbf{\dgal}) which selects samples to maximize training dynamics. 
    In particular, we prove that the convergence speed of training and the generalization performance are positively correlated under the ultra-wide condition and show that maximizing the training dynamics leads to better generalization performance.
    Furthermore, to scale up to large deep neural networks and data sets, we introduce two relaxations for the subset selection problem and reduce the time complexity from polynomial to constant.
    Empirical results show that \dgal\ not only outperforms the other baselines consistently but also scales well on large deep learning models.
    We hope our work would inspire more attempts on bridging the theoretical findings of deep networks and practical impacts of deep active learning in real applications.
\end{abstract}
\section{Introduction}
\label{intro}

Training deep learning (DL) models usually requires large amount of high-quality labeled data~\cite{zhang2016understanding} to optimize a model with a massive number of parameters. The acquisition of such annotated data is usually time-consuming and expensive, making it unaffordable in the fields that require high domain expertise. 
A promising approach for minimizing the labeling effort is active learning (AL), which aims to identify and label the maximally informative samples, so that a high-performing classifier can be trained with minimal labeling effort~\cite{settles2009active}. 
Under classical statistical learning settings, theories of active learning have been extensively studied from the perspective of VC dimension~\cite{hanneke2014theory}. As a result, a variety of methods have been proposed, such as (i) the version-space-based approaches, which require maintaining a set of models~\cite{cohn1994improving, balcan2009agnostic}, and (ii) the clustering-based approaches, which assume that the data within the same cluster have pure labels ~\cite{dasgupta2008hierarchical}. 

However, the theoretical analyses for these classical settings may not hold for over-parameterized deep neural networks where the traditional wisdom is ineffective~\cite{zhang2016understanding}. For example, margin-based methods select the labeling examples in the vicinity of the learned decision boundary~\cite{balcan2007margin, balcan2013active}. However, in the over-parameterized regime, every labeled example could potentially be near the learned decision boundary~\cite{karzand2019active}.
As a result, theoretically, such analysis can hardly guide us to design practical active learning methods. 
Besides, empirically, multiple deep active learning works, borrowing observations and insights from the classical theories and methods, have been observed unable to outperform their passive learning counterparts in a few application scenarios~\cite{kirsch2019batchbald, ash2019deep}.

On the other hand, the analysis of neural network's optimization and generalization performance has witnessed several exciting developments in recent years in terms of the deep learning theory~\cite{jacot2018NTK, arora2019fineGrained, lee2019wide}. 
It is shown that the training dynamics of deep neural networks using gradient descent can be
characterized by the Neural Tangent Kernel (NTK) of infinite~\cite{jacot2018NTK} or finite~\cite{hanin2019finiteNTK} width networks.
This is further leveraged to characterize the generalization of over-parameterized networks through Rademacher complexity analysis~\cite{arora2019fineGrained,cao2019generalization}. We are therefore inspired to ask: How can we design a practical and generic active learning method for deep neural networks with theoretical justifications?

To answer this question, we firstly explore the connection between the model performance on testing data and the convergence speed on training data for the over-parameterized deep neural networks. Based on the NTK framework~\cite{jacot2018NTK, arora2019fineGrained}, we theoretically show that if a deep neural network converges faster (``Train Faster''), then it tends to have better generalization performance (``Generalize Better''), which matches the existing observations \cite{hardt2016train, liu2017algorithmic, lyle2020bayesian, ru2020revisiting, xu2021optimization}. 
Motivated by the aforementioned connection, we first introduce {\it Training Dynamics}, the derivative of training loss with respect to iteration, as a proxy to quantitatively describe the training process. On top of it, we formally propose our generic and theoretically-motivated deep active learning method, \dgal, which will query labels for a subset of unlabeled samples that maximally increase the training dynamics. In order to compute the training dynamics by merely using the unlabeled samples, we leverage two relaxations {\it Pseudo-labeling} and {\it Subset Approximation} to solve this non-trivial subset selection problem. Our relaxed approaches are capable of effectively estimating the training dynamics as well as efficiently solving the subset selection problem by reducing the complexity from $O(N^b)$ to $O(b)$.

In theory, we coin a new term {\it Alignment} to measure the length of the label vector's projection on the neural tangent kernel space. Then, we demonstrate that higher alignment usually comes with a faster convergence speed and a lower generalization bound. Furthermore, with the help of the maximum mean discrepancy~\cite{borgwardt2006integrating}, we extend the previous analysis to an active learning setting where the i.i.d. assumption may not hold. 
Finally, we show that alignment is positively correlated with our active learning goal, training dynamics, which implies that maximizing training dynamics will lead to better generalization performance.

Regarding experiments, we have empirically verified our theory by conducting extensive experiments on three datasets, CIFAR10~\cite{krizhevsky2009learning}, SVHN~\cite{netzer2011reading}, and Caltech101~\cite{fei2004learning} using three types of network structures: vanilla CNN, ResNet~\cite{he2016deep}, and VGG~\cite{simonyan2014very}.
We first show that the result of the subset selection problem delivered by the subset approximation is close to the global optimal solution.
Furthermore, under the active learning setting, our method not only outperforms other baselines but also scales well on large deep learning models. 

The main contributions of our paper can be summarized as follows:
\begin{itemize}[leftmargin=*]
    \item We propose a theory-driven deep active learning method, \dgal, inspired by the observation of ``train faster, generalize better''. To this end, we introduce the Training Dynamics, as a proxy to describe the training process.
    \item We demonstrate that the convergence speed of training and the generalization performance is strongly (positively) correlated under the ultra-wide condition; we also show that maximizing the training dynamics will lead to a lower generalization error in the scenario of active learning.
    \item  Our method is easy to implement. We conduct extensive experiments to evaluate the effectiveness of \dgal\ and empirically show that our method consistently outperforms other methods in a wide range of active learning settings. 
\end{itemize}
\section{Background}
\label{sec:background}
\textbf{Notation.}
We use the random variable $x\in\mathcal{X}$ to represent the input data feature and $y\in\mathcal{Y}$ as the label where $K$ is the number of classes and $[K]:=\{1,2,...,K\}$.
We are given non-degenerated a data source $D$ with unknown distribution $p({x},{y})$.
We further denote the concatenation of $x$ as $X = [x_1, x_2, ..., x_M]^{\top}$ and that of $y$ as ${Y} = [y_1, y_2, ..., y_M]^{\top} $.
We consider a deep learning classifier $h_{\theta}(x) = \text{argmax}~\sigma(f(x; \theta)): x \rightarrow y$ parameterized by $\theta \in \mathbb{R}^p$, where 
$\sigma(\cdot)$ is the softmax function and $f$ is a neural network. 
Let $\otimes$ be the Kronecker Product and $I_{K} \in \mathbb{R}^{K \times K}$ be an identity matrix.
\vspace{-0.3em}

\textbf{Active learning.}
The goal of active learning is to improve the learning efficiency of a model with a limited labeling budget. 
In this work, we consider the pool-based AL setup, where a finite data set $S =\{({x}_l,{y}_l)\}_{l=1}^M$ with $M$ points are $i.i.d.$ sampled from $p({x},{y})$ as the (initial) labeled set.
The AL model receives an unlabeled data set $U$ sampled from $p({x})$ and request labels according to $p({y}|{x})$ for any ${x} \in U$ in each query round. There are $R$ rounds in total, and for each round, a query set $Q$ consisting of $b$ unlabeled samples can be queried. The total budget size $B = b \times R$.

\textbf{Neural Tangent Kernel.}
\label{ntk_intro}
The Neural Tangent Kernel \cite{jacot2018NTK} has been widely  applied to analyze the dynamics of neural networks. If a neural network is sufficiently wide, properly initialized, and trained by gradient descent with infinitesimal step size (\ie, gradient flow), then the neural network is equivalent to kernel regression predictor with a deterministic kernel $\boldsymbol{\Theta}(\cdot, \cdot)$, called Neural Tangent Kernel (NTK). 
When minimizing the mean squared error loss, at the iteration $t$, the dynamics of the neural network $f$ has a closed-form expression:
\begin{equation}
\begin{small}\label{ntk_intro1}
\begin{aligned}
\frac{d f(\mathcal{X}; \theta(t) )}{d t}=- \mathcal{K}_{t} (\mathcal{X}, \mathcal{X})  \left(f(\mathcal{X}; \theta(t))-\mathcal{Y}\right),
\end{aligned}
\end{small}
\end{equation}
where $\theta(t)$ denotes the parameter of the neural network at iteration $t$, $\mathcal{K}_{t} (\mathcal{X}, \mathcal{X})  \in \mathbb{R}^{|\mathcal{X}| \times K \times |\mathcal{X} |\times K}$ is called the empirical NTK and 
$ \mathcal{K}_t^{i,j}({x}, {x'})=\nabla_{\theta} f^i({x}; \theta(t))^{\top} \nabla_{\theta} f^j({x'}; \theta(t))$ 
{is the inner product of the gradient of the $i$-th class probability and the gradient of the $j$-th class probability }
for two samples ${x}, {x'} \in \mathcal{X}$ and $i,j \in [K]$. The time-variant kernel $\mathcal{K}_{t}(\cdot, \cdot)$ is equivalent to the (time-invariant) NTK  with a high probability, i.e., if the neural network is sufficiently wide and properly initialized, then:
\begin{equation}
\begin{small}\label{ntk_intro2}
\begin{aligned}
\mathcal{K}_{t}(\mathcal{X}, \mathcal{X}) = \boldsymbol{\Theta}(\mathcal{X},
\mathcal{X}) \otimes I_{K}.
\end{aligned}
\end{small}
\end{equation}
The final learned neural network at iteration $t$,  is equivalent to the kernel regression solution with respect to the NTK \cite{lee2019wide}. For any input ${x}$ and training data $\{{X},{Y}\}$ we have,
\begin{equation}
\begin{small}\label{ntk_intro3}
\begin{aligned}
f(x; \theta(t)) \approx \boldsymbol{\Theta}({x},{X})^{\top} \boldsymbol{\Theta}({X},{X})^{-1} ({I} - e^{-\eta \boldsymbol{\Theta}({X},{X}) t}) {Y},
\end{aligned}
\end{small}
\end{equation}
where $\eta$ is the learning rate, $\boldsymbol{\Theta}({x},{X})$ is the NTK matrix between input ${x}$ and all samples in training data ${X}$.
\section{Method}
\label{sec:Method}
In section~\ref{subsec:training_dynamics}, we introduce the notion of training dynamics which can be used to describe the training process. Then, in section~\ref{subsec:method}, based on the training dynamics, we propose \dgal. In section~\ref{subsec:discussion}, we discuss the connection between \dgal \ and existing deep active learning methods.

\subsection{Training dynamics}
\label{subsec:training_dynamics}
In this section, we introduce the notion of training dynamics. 
The cross-entropy loss over the labeled set $S$ is defined as:
\begin{equation}
\begin{small}
\begin{aligned}
L(S) &= \sum_{({x}_l, {y}_l) \in S} \ell(f({x}_l; \theta), {y}_l)\\
&= -\sum_{({x}_l, {y}_l) \in S} \sum_{i\in[K]} {y}_l^i \log \sigma^i(f({x}_l; \theta)),
\end{aligned}
\end{small}
\end{equation}
where $\sigma^i(f({x; \theta})) = \frac{\exp(f^i({x; \theta}))}{ \sum_j \exp(f^j( {x; \theta} )) } $.
We first analyze the dynamics of the training loss, with respect to iteration $t$, on one labeled sample (derivation is in Appendix~\ref{deriv_train}):
\begin{equation}
\begin{small}
\begin{aligned}
\frac{\partial \ell (f({x}; \theta), {y})}{\partial t} = -\sum_i \big( y^i - \sigma^i(f({x}; \theta)) \big) \nabla_{\theta}{f^i({x}; \theta)} \nabla^{\top}_{t} \theta .
\label{equ:CE_error_dynamics}
\end{aligned}
\end{small}
\end{equation}
For neural networks trained by gradient descent, if the learning rate $\eta$ is small, then $\nabla_{t} \theta  = \theta_{t+1} - \theta_{t} = -\eta \frac{\partial \sum_{ ({x}_l, {y}_l) \in S } \ell(f({x}_l; \theta), {y}_l)} {\partial \theta}$.
Taking the partial derivative of the training loss with respect to the parameters, we have (the derivation of the following equation can be found in Appendix~\ref{deriv_CE}):
\begin{equation}
\begin{small}
\begin{aligned}
\frac{\partial \ell (f({x}; \theta), {y})}{\partial \theta}  =  \sum_{j \in [K]} \big(\sigma^j(f({x}; \theta)) - {y}^j \big)
\frac{\partial f^j({x}; \theta)}{\partial \theta}. 
\end{aligned}
\end{small}\label{equ:ce_loss_grad}
\end{equation}
Therefore, we can further get the following result for the dynamics of training loss:
\begin{equation}
\begin{small}
\begin{aligned}
& \frac{\partial \ell (f(x; \theta), y)}{\partial t} = - \eta\sum_i \big( \sigma^i(f(x; \theta)) - y^{i} \big) \sum_{j} \sum_{ ({x}_{l^{'}}, {y}_{l^{'}}) \in S } \nabla_{\theta} f^{i}(x; \theta)^{\top} \nabla_{\theta} f^{j}(x_{l^{'}}; \theta) \big(\sigma^j(f(x_{l^{'}}; \theta)) - y_{l^{'}}^j\big).
\end{aligned}
\end{small}
\end{equation}
Furthermore, we define $d^i({X}, {Y}) = \sigma^i(f(X; \theta)) - Y^i$ and $Y^i$ is the label vector of all samples for $i$-th class. 
Then, the {\it training dynamics} (dynamics of training loss) over training set $S${, computed with the empirical NTK $\mathcal{K}^{ij}({X},{X})$,} is denoted by $G(S) \in \mathbb{R}$:
\begin{equation} 
\begin{small}
\begin{aligned}
G(S)  &=  - \frac{1}{\eta} \sum_{ ({x}_l, {y}_l) \in S } \frac{\partial \ell (f(x_l; \theta), y_l)}{\partial t} =  \sum_{i} \sum_{j} {d^i}({X}, {Y})^{\top} \mathcal{K}^{ij}({X},{X}) d^j({X}, {Y}).
\end{aligned}
\end{small}\label{equ:gradient}
\end{equation}

\subsection{Active learning by activating training dynamics}
\label{subsec:method}
Before we present \dgal, we state Proposition~\ref{prop1}, which serves as the theoretical guidance for \dgal \ and will be proved in Section~\ref{theo}.
\begin{prop}\label{prop1}
For deep neural networks,  converging faster leads to a lower worst-case generalization error.
\end{prop}
Motivated by the connection between convergence speed and generalization performance, we propose the general-purpose active learning method, \dgal, which aims to accelerate the convergence by querying labels for unlabeled samples. As we described in the previous section, the training dynamics can be used to describe the training process. Therefore, we employ the training dynamics as a proxy to design an active learning method. Specifically, at each query round, \dgal \ will query labels for samples which maximize the training dynamics $G(S)$, \ie,
\begin{equation}
\begin{small}
\begin{aligned}
Q = \text{argmax}_{Q \subseteq U} G(S \cup \overline{Q} ),~ s.t.~ |Q| = b,
\end{aligned}\label{equ:ce_problem}
\end{small}
\end{equation}
where $\overline{Q}$ is the corresponding data set for $Q$ with ground-truth labels. Notice that when applying the above objective in practice, we are facing two major challenges. 
First, $G(S \cup \overline{Q})$ cannot be directly computed, because the label information of unlabeled examples is not available before the query.
Second, the subset selection problem can be computationally prohibitive if enumerating all possible sets with size $b$. Therefore, we employ the following two relaxations to make this maximization problem to be solved with constant time complexity.

\textbf{Pseudo labeling.} To estimate the training dynamics, we use the predicted label $\hat{{y}}_u$ for sample ${x}_u$ in the unlabeled data set $U$ to compute $G$.
Note, the effectiveness of this adaptation has been demonstrated in the recent gradient-based methods~\cite{ash2019deep, mu2020gradients}, which compute the gradient as if the model’s current prediction on the example is the true label.
Therefore, the maximization problem in Equation~\eqref{equ:ce_problem} is changed to,
\begin{equation}  \label{equ:peseudo_ce_problem}
\begin{small}
\begin{aligned}
 Q = \text{argmax}_{Q \subseteq U} G(S \cup \widehat{Q} ) .
\end{aligned}
\end{small}
\end{equation}
where $\widehat{Q}$ is the corresponding data set for $Q$ with pseudo labels $\widehat{Y}_Q$.

\textbf{Subset approximation.} 
The subset selection problem of Equation~\eqref{equ:peseudo_ce_problem} still requires enumerating all possible subsets of $U$ with size $b$, which is $O(n^b)$. 
We simplify the selection problem to the following problem without causing any change on the result, 
\begin{equation}  
\begin{small}
\begin{aligned}
\text{argmax}_{Q \subseteq U} G(S \cup \widehat{Q} ) =  \text{argmax}_{Q \subseteq U} \Delta(\widehat{Q}|S),
\end{aligned}
\end{small}\label{equ:new_peseudo_ce_problem}
\end{equation}
where $\Delta(\widehat{Q}|S) = G(S \cup \widehat{Q}) - G(S)$ is defined as the change of training dynamics. We approximate the change of training dynamics caused by query set $Q$ using the summation of the change of training dynamics caused by each sample in the query set.
 ~Then the maximization problem can be converted to Equation~\eqref{equ:practical_goal} which can be solved by a greedy algorithm with $O(b)$.
\begin{equation} 
\begin{small}
\begin{aligned}
Q  =  \text{argmax}_{Q \subseteq U} \sum_{(x, \widehat{y}) \in \widehat{Q}} \Delta(\{(x, \widehat{y}) \}|S),~ s.t.~ |Q| = b.
\end{aligned}\label{equ:practical_goal}
\end{small}
\end{equation} 
To further show the approximated result is reasonably good, we decompose the change of training dynamics as (derivation in Appendix~\ref{training_dynamics_decomposition}):
\begin{equation} \label{equ:delta_decompose}
\textcolor{black}{
\begin{small}
\begin{aligned}
&\Delta(\widehat{Q}|S) = \sum_{(x, \widehat{y}) \in \widehat{Q}}  \Delta(\{(x, \widehat{y}) \}|S) + \sum_{(x, \hat{y}), (x', \hat{y}') \in \widehat{Q}} {d^i}(x, \hat{y})^{\top} \mathcal{K}^{ij}(x,x') {d^j}(x', \hat{y}'),
\end{aligned}
\end{small}
}
\end{equation}
\textcolor{black}{where $\mathcal{K}^{ij}(x,x')$ is the empirical NTK.}
The first term in the right hand side is the approximated change of training dynamics. Then, we further define the {\it Approximation Ratio}~\eqref{equ:approx_ratio} which measures the approximation quality,
\begin{equation}
\begin{small}
\begin{aligned}
R(\widehat{Q}|S) = \frac{\sum_{(x, \widehat{y}) \in \widehat{Q}}  \Delta(\{(x, \widehat{y}) \}|S)}{\Delta(\widehat{Q}|S)}.
\end{aligned}
\end{small}\label{equ:approx_ratio}
\end{equation}
We empirically measure the expectation of the Approximation Ratio on two data sets with two different neural networks under three different batch sizes. As shown in Figure~\ref{fig:self_term}, the expectation $\mathbb{E}_{Q \sim U} R(\widehat{Q}|S) \approx 1$ when the model is converged. Therefore, the approximated result delivered by the greedy algorithm is close to the global optimal solution of the original maximization problem, Equation~\eqref{equ:peseudo_ce_problem},  especially when the model is converged.

Based on the above two approximations, we present the proposed method \dgal\ in Algorithm~\ref{alg:main}. As described below, the algorithm starts by training a neural network $f(\cdot; \theta)$ on the initial labeled set $S$ until convergence. Then, for every unlabeled sample $x_u$, we compute pseudo label $\hat{y}_u$ and the change of training dynamics $\Delta(\{(x_u, \widehat{y}_u) \}|S)$. After that, \dgal\ will query labels for top-$b$ samples causing the maximal change on training dynamics, train the neural network on the extended labeled set, and repeat the process. {\color{black} Note, to keep close to the theoretical analysis, re-initialization is not used after each query, which also enables \dgal\ to get rid of the computational overhead of retraining the deep neural networks every time.}

\begin{algorithm}[H] 
\caption{  Deep Active Learning by Leveraging Training Dynamics }
\begin{algorithmic}
\label{alg:main}
\STATE \textbf{Input:} Neural network $f(\cdot; \theta)$, unlabeled sample set $U$, initial labeled set $S$, number of query round $R$, query batch size $b$.\\
\FOR{$r=1$ \KwTo $R$}
  \STATE Train $f(\cdot; \theta)$ on $S$ with cross-entropy loss until convergence. \\
  \FOR{$x_u \in U$} 
      \STATE Compute its pseudo label $\hat{y}_u = \text{argmax} f(x_u; \theta) $.\\
    \STATE Compute $\Delta\left(\{(x_u, \widehat{y}_u) \}|S\right)$.\\
   \ENDFOR
  \STATE Select $b$ query samples $Q$ with the highest $\Delta$ values, and request their labels from the oracle. \\
  \STATE Update the labeled data set $S = S \cup \overline{Q} $ .
\ENDFOR\\
\textbf{return} Final model $f(\cdot; \theta)$.
\end{algorithmic}
\end{algorithm}

\subsection{Relation to existing methods}
\label{subsec:discussion}
Although existing deep active learning methods are usually designed based on heuristic criteria, some of them have empirically shown their effectiveness~\cite{ash2019deep, liu2021influence, huang2016active}. We surprisingly found that our theoretically-motivated method \dgal\ has some connections with those existing methods from the perspective of active learning criterion. The proposed active learning criterion in Equation~\eqref{equ:practical_goal} can be explicitly written as (derivation in Appendix~\ref{training_dynamics_simplification}):
\begin{equation}
\begin{small}
\begin{aligned}
  \Delta(\{(x_u,& \widehat{y}_u) \}| S) =  
  \|\nabla_{\theta} \ell (f({x}_u; \theta), \hat{{y}}_u)\|^2 + 2 \sum_{(x, y) \in S} \nabla_{\theta}\ell(f({x}_u; \theta), \hat{{y}}_u)^{\top} \nabla_{\theta} \ell (f({x}; \theta), {{y}}).
\end{aligned}
\end{small}\label{equ:discussion}
\end{equation}
\color{black}
\textbf{Note.} The first term of the right-hand side can be interpreted as the square of gradient length (2-norm) which reflects the 
uncertainty of the model on the example and has been wildly used as an active learning criterion in some existing works~\cite{huang2016active, ash2019deep, shukla2021egl++}. The second term can be viewed as the influence function~\cite{koh2017understanding} with identity hessian matrix. And recently, ~\cite{liu2021influence} has empirically shown that the effectiveness of using the influence function with identity hessian matrix as active learning criterion. We hope our theoretical analysis can also shed some light on the interpretation of previous methods.
\section{Theoretical analysis}
\label{theo}

In this section, 
we study the correlation between the convergence rate of the training loss and the generalization error under the ultra-wide condition~\cite{jacot2018NTK,arora2019fineGrained}.
We define a measure named {\it alignment} to quantify the convergence rate and further show its connection with generalization bound. The analysis provides a theoretical guarantee for the phenomenon of ``Train Faster, Generalize Better'' as well as our active learning method \dgal\ with a rigorous treatment. Finally, we show that the active learning proxy, training dynamics, is correlated with alignment, which indicates that increasing the training dynamics leads to larger convergence rate and better generalization performance. We leave all proofs of theorems and details of verification experiments in Appendix~\ref{proofs} and~\ref{ultra-exp} respectively.

\subsection{Train faster provably generalize better} 

Given an ultra-wide neural network, the gradient descent can achieve a near-zero training error \cite{jacot2018NTK,arora2019exact} and its generalization ability in unseen data can be bounded \cite{arora2019fineGrained}. It is shown that both the convergence and generalization {of a neural network} can be analyzed using the NTK \cite{arora2019fineGrained}. However, the question what is the relation between the convergence rate and the generalization bound has not been answered. We formally give a solution by introducing the concept of {\it alignment}, which is defined as follows: 

\begin{defi} [Alignment]
Given a data set $ S = \{X,Y\}$, the alignment is a measure of correlation between $X$ and $Y$ projected in the NTK space. In particular, the alignment can be computed by $\mathcal{A}(X,Y) = \Tr [Y^{\top} \boldsymbol{\Theta}(X,X)  Y]  = \sum_{k=1}^K \sum_{i=1}^n \lambda_i (\Vec{v}^{\top}_i Y^k)^2$.
\end{defi}

In the following, we will demonstrate why ``Train Faster'' leads to ``Generalize Better'' through alignment. In particular, the relation of the convergence rate and the generalization bound with alignment is analyzed. The convergence rate of gradient descent for ultra-wide networks is presented in following lemma:
\begin{lem} [Convergence Analysis with NTK, Theorem 4.1 of \cite{arora2019fineGrained}] \label{lem:convergence}
Suppose $\lambda_0 = \lambda_{\min}(\boldsymbol{\Theta}) > 0$ for all subsets of data samples. For $\delta \in (0,1)$, if $m = \Omega(\frac{n^7}{\lambda^4_0  \delta^4 \epsilon^2})$ and $\eta = O(\frac{\lambda_0}{n^2})$, with probability at least $1-\delta$, the network can achieve near-zero training error,
\begin{equation}
\begin{small}
\begin{aligned} 
\norm{{Y}-f({X; \theta(t)})}_2 = \sqrt{\sum_{k=1}^K \sum_{i=1}^{n} (1-\eta \lambda_i)^{2t} (\Vec{v}^{\top}_i {Y}^k)^2} \pm \epsilon,
\end{aligned}
\end{small}\label{equ:conveg}
\end{equation}
where $n$ denotes the number of training samples and $m$ denotes the width of hidden layers. The NTK $\boldsymbol{\Theta} = V^{\top} \Lambda V$ with $\Lambda = \{\lambda_i \}_{i=1}^n$ is a diagonal matrix of eigenvalues and $V =\{\Vec{v}_i \}_{i=1}^n$ is a unitary matrix.
\end{lem}
In this lemma, we take mean square error (MSE) loss as an example for the convenience of illustration. The conclusion can be extended to other loss functions such as cross-entropy loss (see Appendix B.2 in \cite{lee2019wide}). From the lemma, we find the convergence rate is governed by the dominant term (\ref{equ:conveg}) as $\mathcal{E}_t({X},{Y}) =  \sqrt{\sum_{k=1}^K \sum_{i=1}^{n} (1-\eta \lambda_i)^{2t} (\Vec{v}^{\top}_i {Y}^k)^2}$, which is correlated with the {\it alignment}:

\begin{thm} [Relationship between the convergence rate and alignment] \label{thm1}
Under the same assumptions as in Lemma \ref{lem:convergence}, the convergence rate described by $\mathcal{E}_t$ satisfies, 
\begin{equation}
\begin{small}
\begin{aligned} 
\Tr [Y^{\top} Y] - 2t\eta \mathcal{A}({X},Y) \le \mathcal{E}^2_t(X,Y) \le \Tr[Y^{\top} Y] -\eta \mathcal{A}({X},Y).
\end{aligned}
\end{small}
\end{equation}
\end{thm}
\begin{rem}
In the above theorem, we demonstrate that the alignment can measure the convergence rate. Especially, we find that both the upper bound and the lower bound of error $\mathcal{E}_t({X},Y)$ are inversely proportional to the alignment, which implies that higher alignment will lead to achieving faster convergence. 
\end{rem}


Now we analyze the generalization performance of the proposed method through complexity analysis. We demonstrate that the ultra-wide networks can achieve a reasonable generalization bound.
\begin{lem} [Generalization bound with NTK, Theorem 5.1 of \cite{arora2019fineGrained}] \label{lem:generalization} Suppose data $S = \{ (x_i,y_i)\}_{i=1}^n$ are i.i.d. samples from a non-degenerate distribution $p(x,y)$, and $m \ge {\rm poly}(n, \lambda_0^{-1}, \delta^{-1})$. Consider any loss function $\ell: \mathbb{R} \times \mathbb{R} \rightarrow [0,1]$ that is $1$-Lipschitz, then with probability at least $1-\delta$ over the random initialization, the network trained by gradient descent for $T \ge \Omega(\frac{1}{\eta \lambda_0} \log \frac{n}{\delta})$ iterations has population risk $\mathcal{L}_p = \mathbb{E}_{(x,y)\sim p(x,y)}[\ell(f_T(x; \theta),y)]$ that is bounded as follows:
\begin{equation}
\begin{small}
\begin{aligned}
   \mathcal{L}_p \le \sqrt{\frac{2 \Tr[{Y^{\top} \boldsymbol{\Theta}^{-1}(X, X) Y}]}{n}}+O \bigg(\sqrt{\frac{\log \frac{n}{\lambda_0 \delta}}{n}} \bigg).
\end{aligned}
\end{small}
\end{equation}
\end{lem}
In this lemma, we show that the dominant term in the generalization upper bound is $\mathcal{B}(X,Y) = \sqrt{\frac{2 \Tr[{Y^{\top} \boldsymbol{\Theta}^{-1} Y}]}{n}}$. In the following theorem, we further prove that this bound is inversely proportional to the alignment $\mathcal{A}(X,Y)$.

\begin{thm} [Relationship between the generalization bound and alignment] \label{thm2}
Under the same assumptions as in Lemma \ref{lem:generalization},
 if we define the generalization upper bound as $\mathcal{B}(X,Y) = \sqrt{\frac{2 \Tr[{Y^{\top} \boldsymbol{\Theta}^{-1} Y}]}{n}}$, then it can be bounded with the alignment as follows:
\begin{equation}
\begin{small}
\label{equ:alignment_bound}
\begin{aligned}
\frac{ \Tr^2[Y^{\top} Y]}{ \mathcal{A}({X},Y)} \le \frac{n}{2}\mathcal{B}^2({X},Y) \le \frac{\lambda_{max}}{\lambda_{min}} \frac{  \Tr^2 [Y^{\top} Y] }{\mathcal{A}({X},Y) } .
\end{aligned}
\end{small}
\end{equation}             
\end{thm}

\begin{rem}
Theorems \ref{thm1} and \ref{thm2} reveal that the cause for the correlated phenomenons ``Train Faster'' and ``Generalize Better'' is the projection of label vector on the NTK space (alignment). 
\end{rem}

\subsection{`` Train Faster, Generalize Better '' for active learning}
\begin{wrapfigure}{r}{5cm}
\vspace{-20pt}
\hspace{-5pt}
\includegraphics[width=5.1cm]{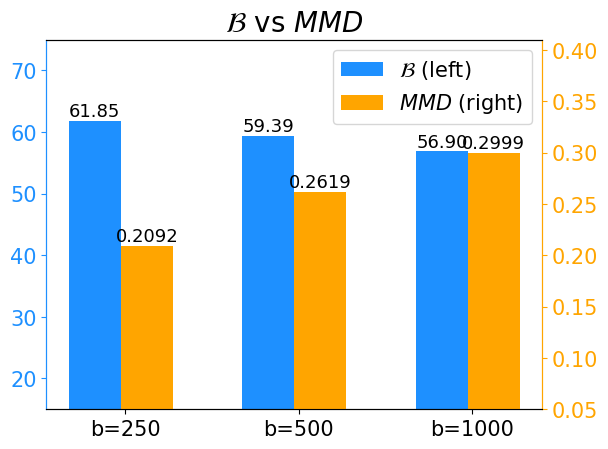}
\caption{Comparison between Empirical Generalization Bound and MMD.}
\label{fig:MMD_B}
\vspace{-30 pt}
\end{wrapfigure}
In the NTK framework~\cite{arora2019fineGrained}, the empirical average requires data in $S$ is \textit{i.i.d.} samples (Lemma \ref{lem:generalization}). However, this assumption may not hold in the active learning setting with multiple query rounds, because the training data is composed by \textit{i.i.d.} sampled initial label set and samples queried by active learning policy. To extend the previous analysis principle to active learning, we follow \cite{wang2015querying} to reformulate the Lemma \ref{lem:generalization} as:
\begin{equation}
\begin{small}
\begin{aligned}
   &\mathcal{L}_{{p}} \le ( \mathcal{L}_{{p}} - \mathcal{L}_{{q}} )
   + \sqrt{\frac{2 \Tr[{Y^{\top} \boldsymbol{\Theta}^{-1}(X, X) Y}]}{n}}
   +O \bigg(\sqrt{\frac{\log \frac{n}{\lambda_0 \delta}}{n}} \bigg),
\end{aligned}
\end{small}\label{lemma2_active}
\end{equation}  
where $\mathcal{L}_{{q}}=\mathbb{E}_{(x,y)\sim q(x,y)}[\ell(f(x; \theta),y)]$, $q(x,y)$ denotes the data distribution after query, and $X, Y$ includes initial training samples and samples after query. There is a new term in the upper bound, which is the difference between the true risk under different data distributions.
\begin{equation}
\begin{small}
\begin{aligned}
\mathcal{L}_{{p}} - \mathcal{L}_{{q}}  =& \mathbb{E}_{(x,y)\sim p(x,y)}[\ell(f(x; \theta),y)] - \mathbb{E}_{(x,y)\sim q(x,y)}[\ell(f(x; \theta),y)]
\end{aligned}
\end{small}
\end{equation}
Though in active learning the data distribution for the labeled samples may be different from the original distribution, they share the same conditional probability $p(y | x)$. We define $g(x) = \int_{y} \ell(f({x; \theta}), y) p(y|x) dy$, and then we have: 
\begin{equation}
\begin{small}
\begin{aligned}
\mathcal{L}_{{p}} - \mathcal{L}_{{q}} = \int_{{x}} g({x}) p({x}) d {x}-\int_{{x}} g({x}) q({x}) d {x} .
\end{aligned}
\end{small}
\end{equation}
To measure the distance between two distributions, we employ the Maximum Mean Discrepancy (MMD) with neural tangent kernel~\cite{jia2021efficient} (derivation in Appendix~\ref{derivation_mmd}).
\begin{equation}
\begin{small}
\begin{aligned}
\mathcal{L}_{{p}} - \mathcal{L}_{{q}} \leq \text{MMD}( S_{0}, S, \mathcal{H}_{\boldsymbol{\Theta}} )  + O\Big(\sqrt{\frac{C \ln (1 / \delta)}{n}}\Big).
\end{aligned}\label{mmd}
\end{small}
\end{equation}
\color{black}
Slightly overloading the notation, we denote the initial labeled set as $S_{0}$, $\mathcal{H}_{\boldsymbol{\Theta}}$ as the associated Reproducing Kernel Hilbert Space for the NTK $\boldsymbol{\Theta}$,  and $\forall x, x' \in S, \boldsymbol{\Theta}(x,x') \leq C$. Note, $\text{MMD}( S_{0}, S, \mathcal{H}_{\boldsymbol{\Theta}} )$ is the empirical measure for $\text{MMD}( p(x), q(x), \mathcal{H}_{\boldsymbol{\Theta}} )$.
We empirically compute MMD and the dominant term of the generalization  upper  bound $\mathcal{B}$ under the active learning setting with our method \dgal. 
As shown in Figure \ref{fig:MMD_B}, on CIFAR10 with a CNN target model (three convolutional layers with global average pooling), the initial labeled set size $|S| = 500$, query round $R=1$ and budget size $b \in \{250, 500, 1000\}$, we observe that, under different active learning settings, the MMD is always much smaller than the $\mathcal{B}$. 
{\color{black}Besides, we further investigate the MMD and $\mathcal{B}$ for $R\geq2$ and observe the similar results.
Therefore, the lemma \ref{lem:generalization} still holds for the target model with \dgal. 
More results and discussions for $R\geq2$ are in Appendix~\ref{mmd_multi_round} and the computation details of MMD and NTK are in Appendix~\ref{ntk_mmd}.}

\newpage
\subsection{Alignment and training dynamics in active learning}
\begin{wrapfigure}{r}{4.7cm}
\hspace{-10pt}
\includegraphics[width=4.5cm]{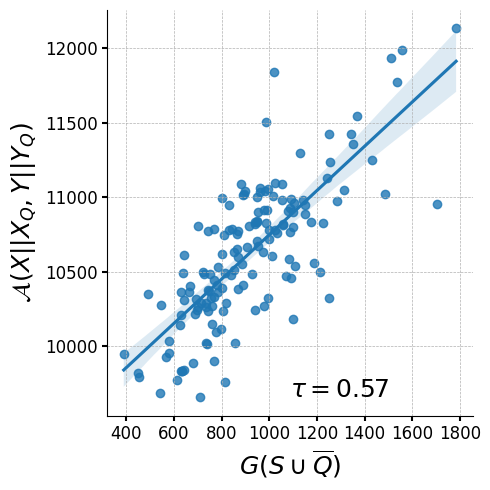}
\vspace{-10pt}
\caption{Relation between Alignment and Training Dynamics.}
\label{fig:align_dynamics}
\vspace{-10pt}
\end{wrapfigure}
In this section, we show the relationship between the alignment and the training dynamics. To be consistent with the previous theoretical analysis (Theorem~\ref{thm1} and~\ref{thm2}), we use the training dynamics with mean square error under the ultra-width condition, which can be expressed as $G_{MSE}(S) =  \Tr{ \big[ (f(X; \theta)-Y)^{\top}\boldsymbol{\Theta}(X,X) (f(X; \theta)-Y) \big] }$. Due to the limited space, we leave the derivation in Appendix~\ref{mse_gradient}. 
To further quantitatively evaluate the correlation between $G_{MSE}(S \cup \overline{Q})$ and $\mathcal{A}(X \| X_Q, Y \| Y_Q)$, we utilize the Kendall $\tau$ coefficient~\cite{kendall1938new} to empirically measure their relation. 
As shown in Figure~\ref{fig:align_dynamics}, for CNN on CIFAR10 with active learning setting, where $|S|=500$ and $|\overline{Q}|=250$, there is a strong agreement between $G_{MSE}(S \cup \overline{Q})$ and $\mathcal{A}(X \| X_Q, Y \| Y_Q)$, which further indicates that increasing the training dynamics will lead to a faster convergence and better generalization performance.
More details about this verification experiment are in Appendix~\ref{dynamics_alignment}.
\section{Experiments}
\label{exp}

\subsection{Experiment setup}
\vspace{-2 pt}

\textbf{Baselines.} We compare \dgal\ with the following eight baselines: Random, Corset, Confidence  Sampling~(Conf), Margin Sampling~(Marg), Entropy, and Active Learning by Learning~(ALBL), Batch Active learning by Diverse Gradient Embeddings~(BADGE). Description of baseline methods is in Appendix~\ref{baselines}.

\textbf{Data sets and Target Model.} We evaluate all the methods on three benchmark data sets, namely, CIFAR10~\cite{krizhevsky2009learning}, SVHN~\cite{netzer2011reading}, and Caltech101~\cite{fei2004learning}. We use accuracy as the evaluation metric and report the mean value of 5 runs. We consider three neural network architectures: vanilla CNN, ResNet18~\cite{he2016deep}, and VGG11~\cite{simonyan2014very}. For each model, we keep the hyper-parameters used in their official implementations. More information about the implementation is in Appendix~\ref{model_archi}.

\textbf{Active Learning Protocol.} 
Following the previous evaluation protocol~\cite{ash2019deep}, we compare all those active learning methods in a batch-mode setup with an initial set size $M=500$ for all those three data sets, batch size $b$ varying from $\{250, 500, 1000\}$. 
\textcolor{black}{For the selection of test set, we use the benchmark split of the CIFAR10~\cite{krizhevsky2009learning}, SVHN~\cite{netzer2011reading} and sample 20\% from each class to form the test set for the Caltech101~\cite{fei2004learning}.}

\subsection{Results and analysis}
\vspace{-2 pt}
The main experimental results have been provided as plots due to the limited space. We also provide tables in which we report the mean and standard deviation for each plot in Appendix~\ref{exp_number}.

\textbf{Overall results.}
The average test accuracy at each query round is shown in Figure~\ref{fig:main}. Our method \dgal\ can consistently outperform other methods for all query rounds. This suggests that \dgal\ is a good choice regardless of the labeling budget.
And, we notice \dgal\ can work well on data sets with a large class number, such as Caltech101. However, the previous state-of-the-art method, BADGE, cannot be scaled up to those data sets, because the required memory is linear with the number of classes. 
{\color{black} Besides, because \dgal\ depends on pseudo labeling, a relatively large initial labeled set can provide advantages for \dgal. Therefore, it is important to examine whether \dgal\ can work well with a small initial labeled set.
As shown in Figure~\ref{fig:main}, \dgal\ is able to work well with a relatively small initial labeled set ($M=500$).}
Due to the limited space, we only show the result under three different settings in Figure~\ref{fig:main}. More evaluation results are in Appendix~\ref{rest_main_exp}. 
{\color{black} Moreover, although the re-initialization trick makes \dgal\ deviate from the dynamics analysis, we investigate the effect of it to \dgal\ and provide the empirical observations and analysis in Appendix~\ref{retrain_exp}. }
\begin{figure*}[h] 
\centering
\vspace{-10 pt}
\hspace{-10pt}
\begin{minipage}{0.28\textwidth}
\includegraphics[width =1.8in]{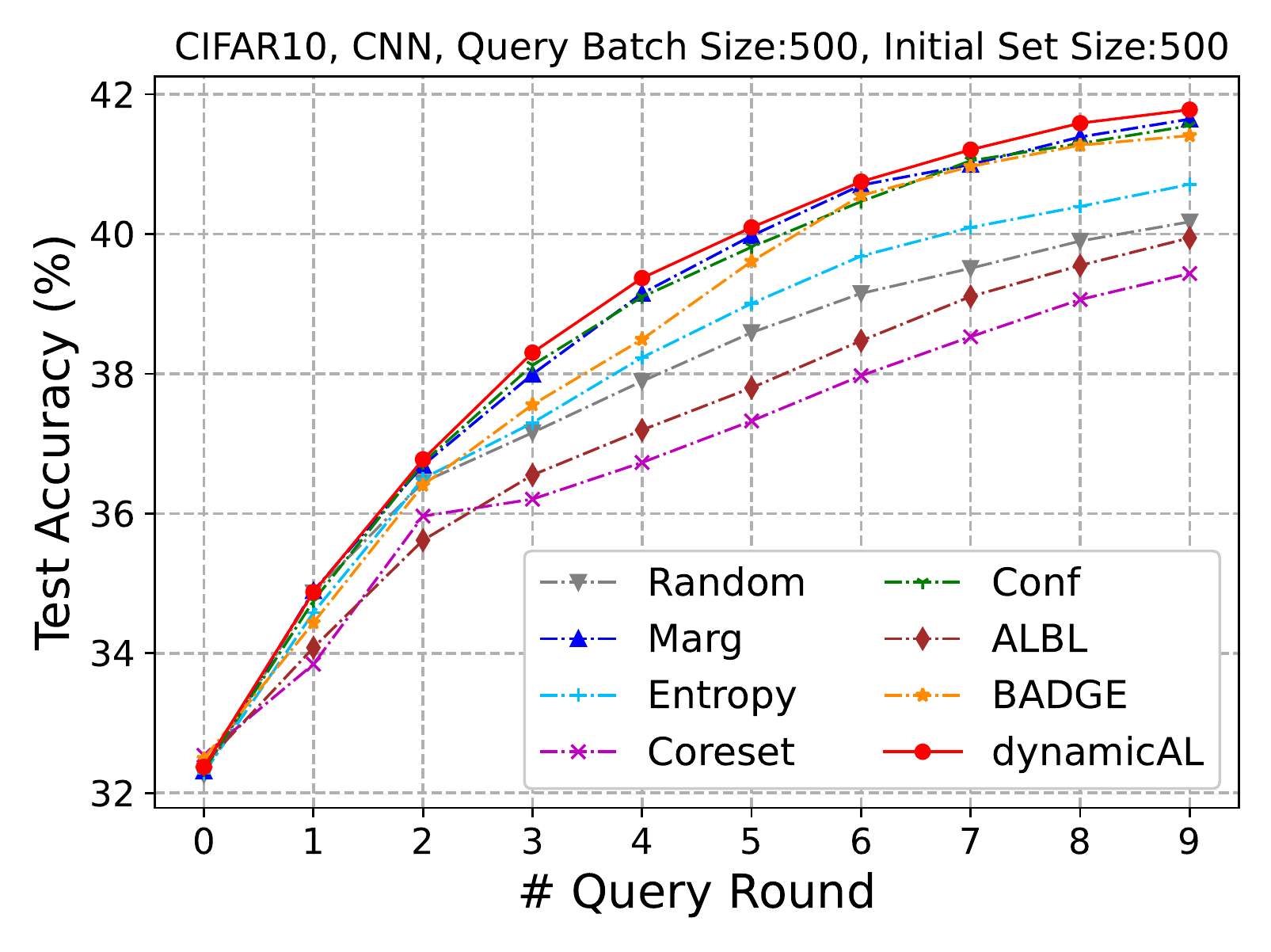}
\end{minipage} 
\hspace{15pt}
\begin{minipage}{0.28\textwidth}
\includegraphics[width =1.8in]{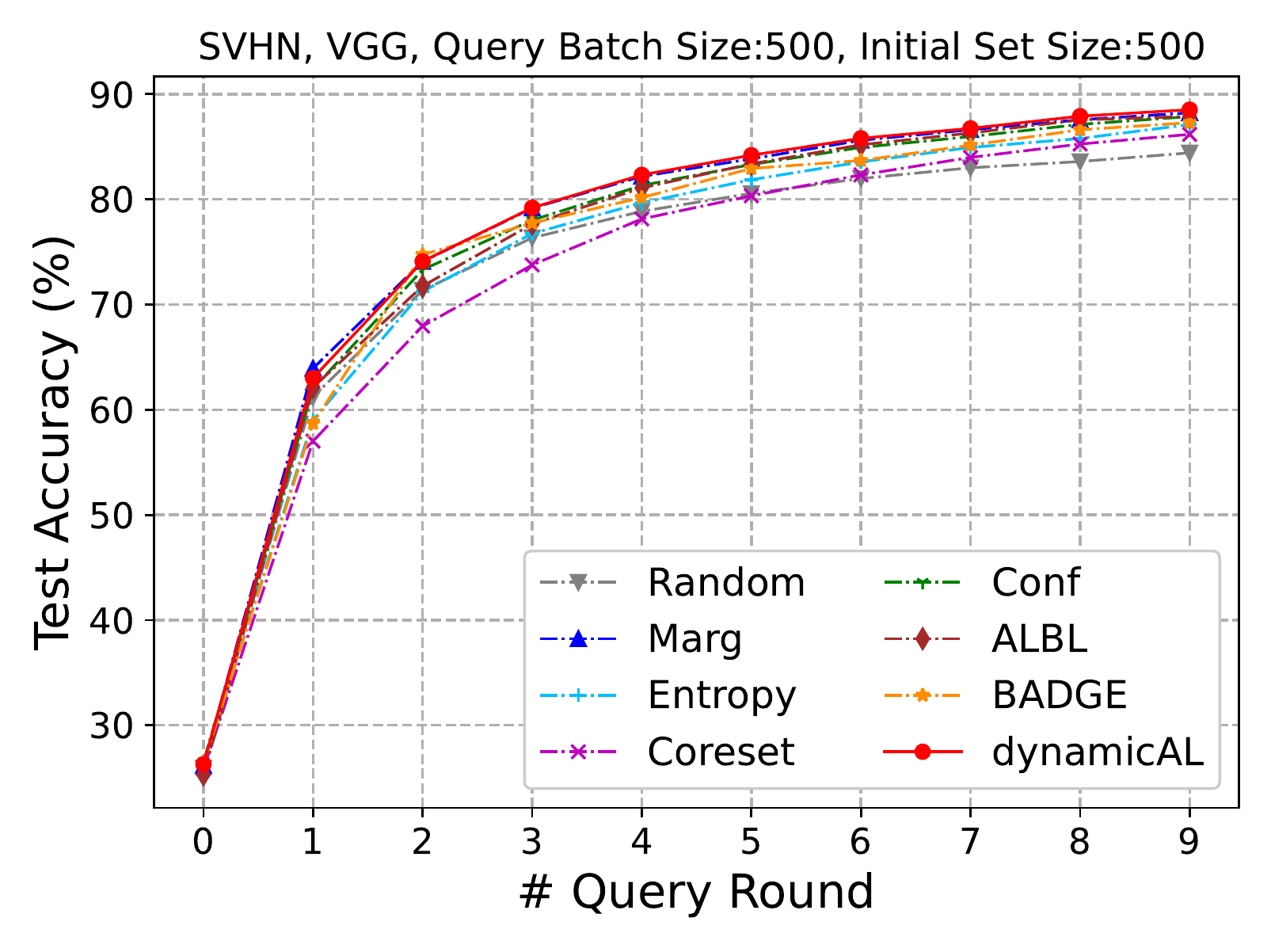}
\end{minipage}
\hspace{15pt}
\begin{minipage}{0.28\textwidth}
\includegraphics[width =1.8in]{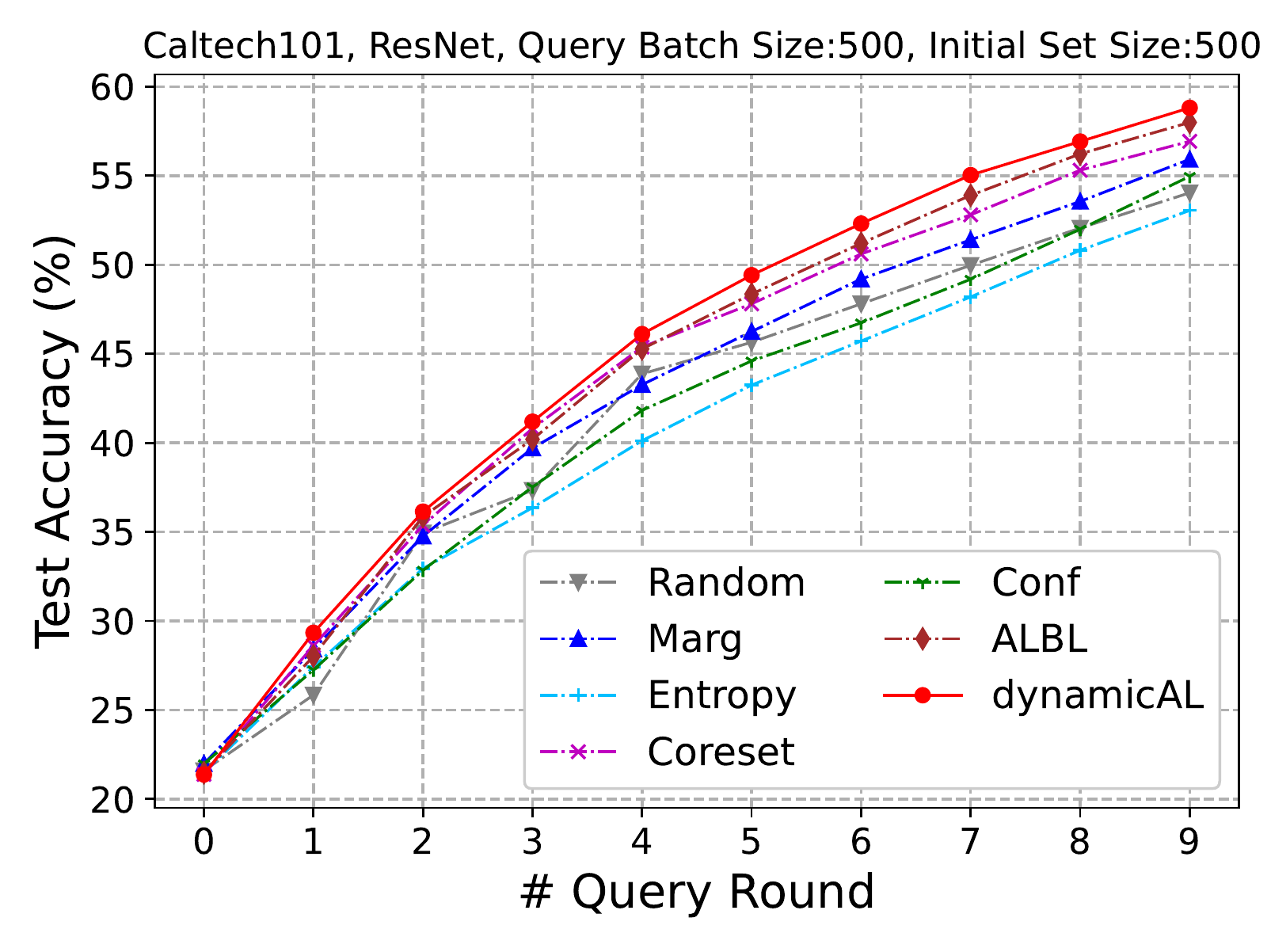}
\end{minipage}
\vspace{-3 pt}
\caption{Active learning test accuracy versus the number of query rounds for a range of conditions.}
\label{fig:main}
\vspace{-6 pt}
\end{figure*}

\textbf{Effect of query size and query round.}
Given the total label budget $B$, the increasing of query size always leads to the decreasing of query round. We study the influence of different query size and query round on \dgal\ from two perspectives. First, we study the expected approximation ratio with different query batch sizes on different data sets. As shown in Figure~\ref{fig:self_term}, under different settings the expected approximation ratio always converges to 1 with the increase of training epochs, which further indicates that the query set selected by using the approximated change of training dynamics is a reasonably good result for the query set selection problem.
Second, we study influence of query round for actual performance of target models. The performance for different target models on different data sets with total budge size $B=1000$ is shown in Table~\ref{tab:batch_size}. For certain query budget, our active learning algorithm can be further improved if more query rounds are allowed. 
\begin{figure}[h!] 
\centering
\hspace{-30pt}
\begin{minipage}{0.32\textwidth}
\includegraphics[width =2.1in]{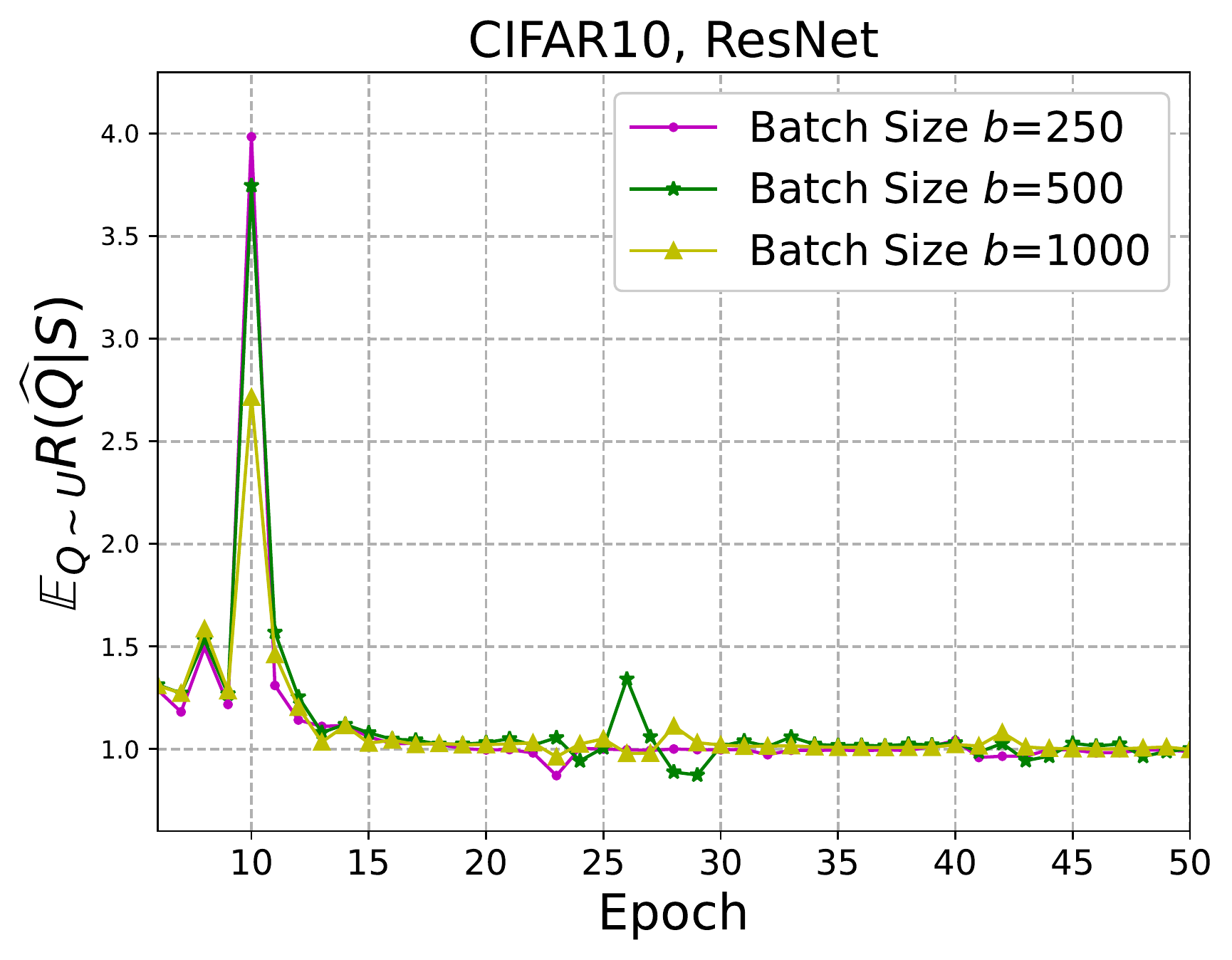}
\end{minipage}
\hspace{30pt}
\begin{minipage}{0.32\textwidth}
\includegraphics[width =2.1in]{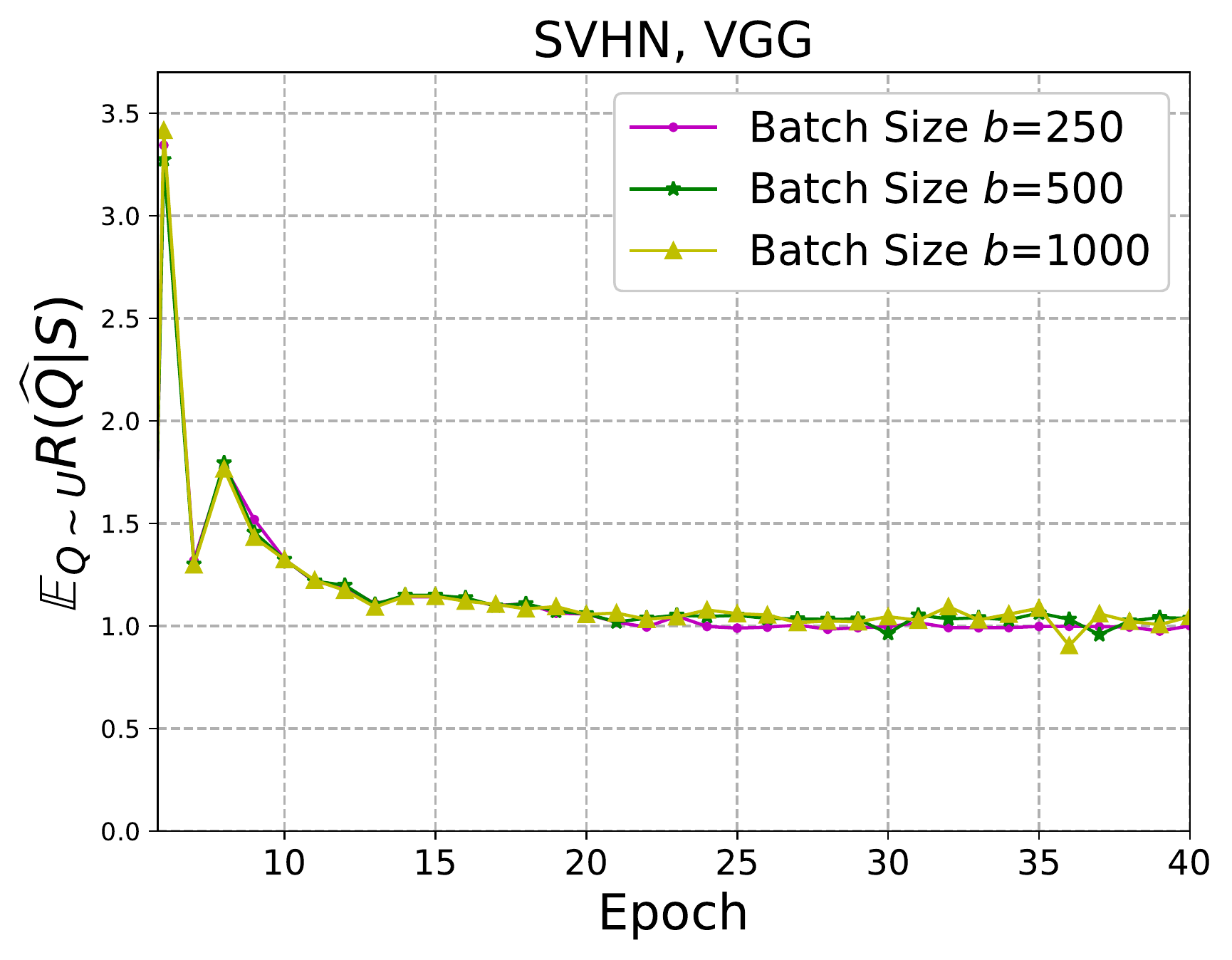}
\end{minipage}
\vspace{-3 pt}
\caption{{The Expectation of the Approximation Ratio with different query batch sizes $b$.}}
\vspace{-10 pt}
\label{fig:self_term}
\end{figure}
\begin{table}[!h]
\vspace{-7pt} 
\centering
\caption{Accuracy of \dgal \ with different query batch size $b$.}
\vspace{4 pt}
\setlength{\tabcolsep}{4pt}{
\begin{small}
\begin{tabular}{ccccc}
\toprule
\textbf{Setting} &\textbf{CIFAR10+CNN}&\textbf{CIFAR10+Resnet} &\textbf{SVHN+VGG} &\textbf{Caltech101+Resnet} \\
\midrule
$R=10, b=100$     & \textbf{36.84} & \textbf{40.92} & \textbf{76.34} & \textbf{37.06} \\
$R=4, b=250$   & 36.72 & 40.78 & 75.26 & 36.48   \\
$R=2, b=500$   & 36.71 & 40.46 & 74.10 & 35.91  \\
$R=1, b=1000$   & 36.67 & 40.09 & 70.04 & 33.82   \\
\bottomrule
\label{tab:batch_size}
\end{tabular}
\end{small}
}
\vspace{-10pt}
\end{table}

\textbf{Comparison with different variants.}
The active learning  criterion of \dgal\ can be written as $ \sum_{(x, y) \in S} 
\|\nabla_{\theta} \ell (f({x; \theta}_u), \hat{{y}}_u)\|^2 + \gamma \nabla_{\theta}\ell(f({x}_u; \theta), \hat{{y}}_u)^{\top} \nabla_{\theta} \ell (f({x}; \theta), {{y}})$. We empirically show the performance for $\gamma \in \{0,1,2, \infty\}$ in Figure~\ref{fig:variants}. 
With $\gamma=0$, the criterion is close to the expected gradient length method~\cite{shukla2021egl++}. And with $\gamma=\infty$, the selected samples are same with the samples selected by using the influence function with identity hessian matrix criterion~\cite{liu2021influence}.
As shown in Figure~\ref{fig:variants}, the model achieves the best performance with $\gamma=2$, which is aligned with the value indicated by the theoretical analysis (Equation~\ref{equ:discussion}).
The result confirms the importance of theoretical analysis for the design of deep active learning methods.
\begin{figure}[ht]
\vspace{-10 pt}
\hspace{-12pt}
\begin{minipage}{0.28\textwidth}
\includegraphics[width = 1.9in]{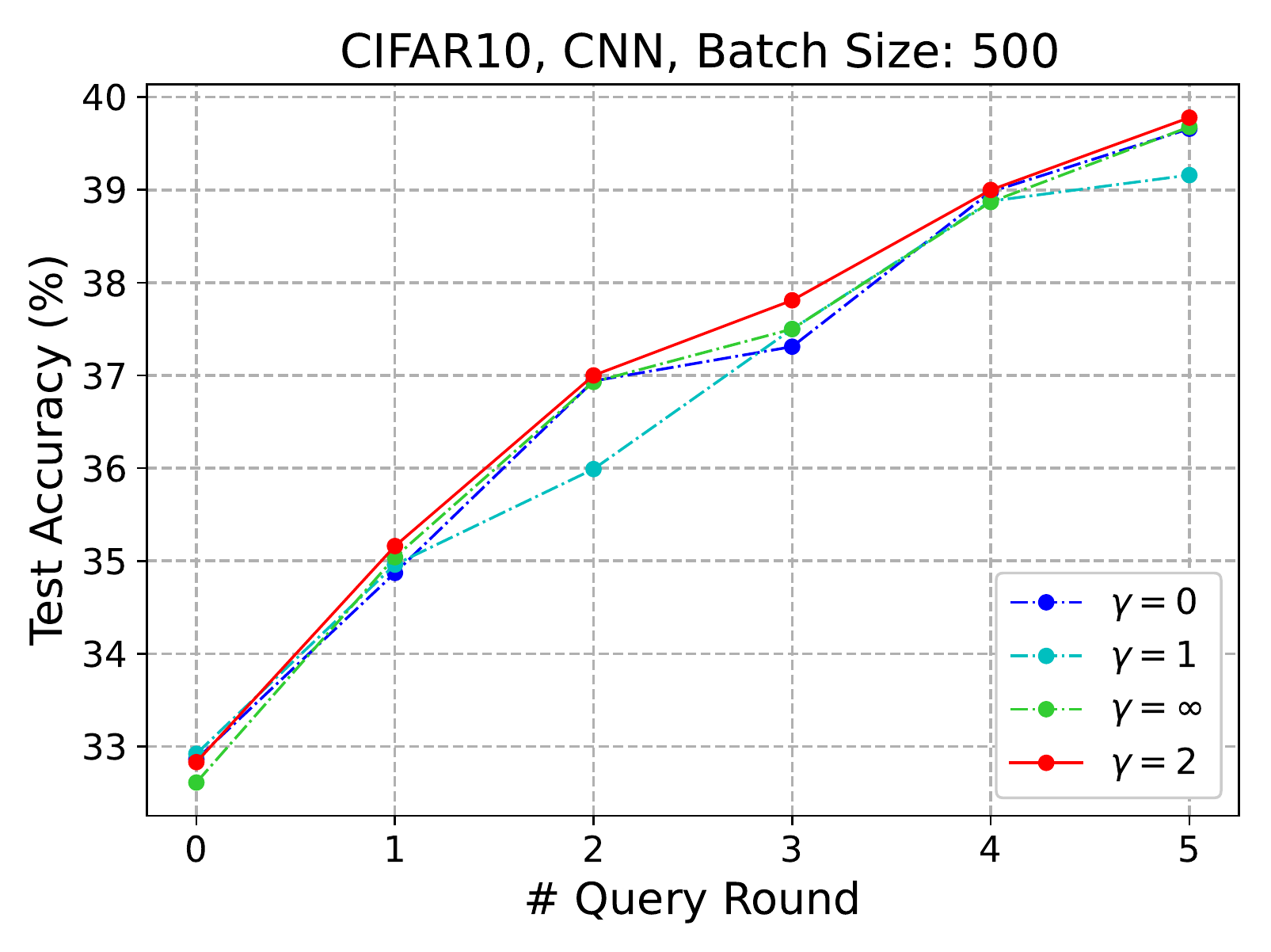}
\end{minipage}
\hspace{20pt}
\begin{minipage}{0.28\textwidth}
\includegraphics[width = 1.9in]{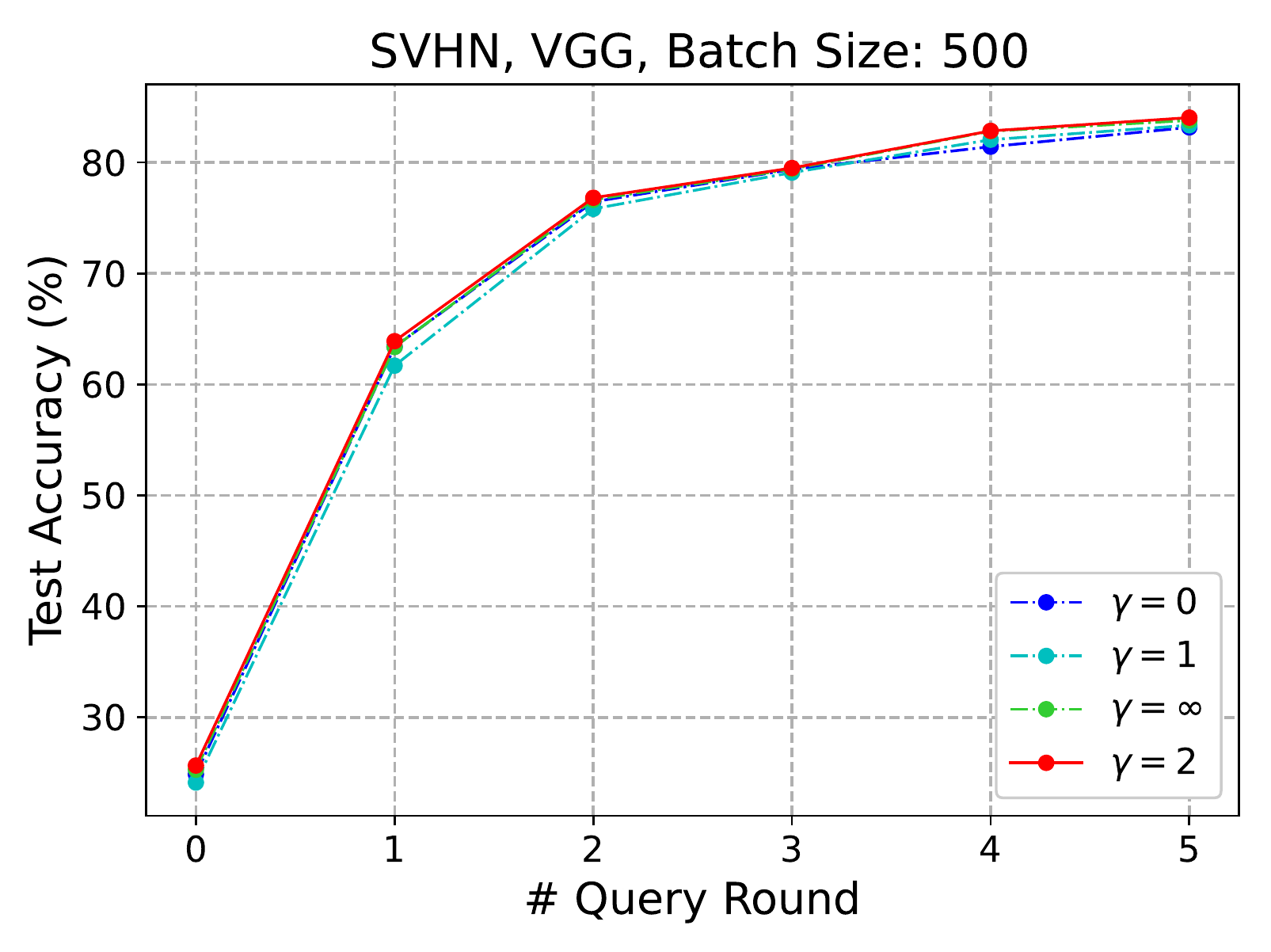}
\end{minipage}
\hspace{20pt}
\begin{minipage}{0.28\textwidth}
\includegraphics[width = 1.9in]{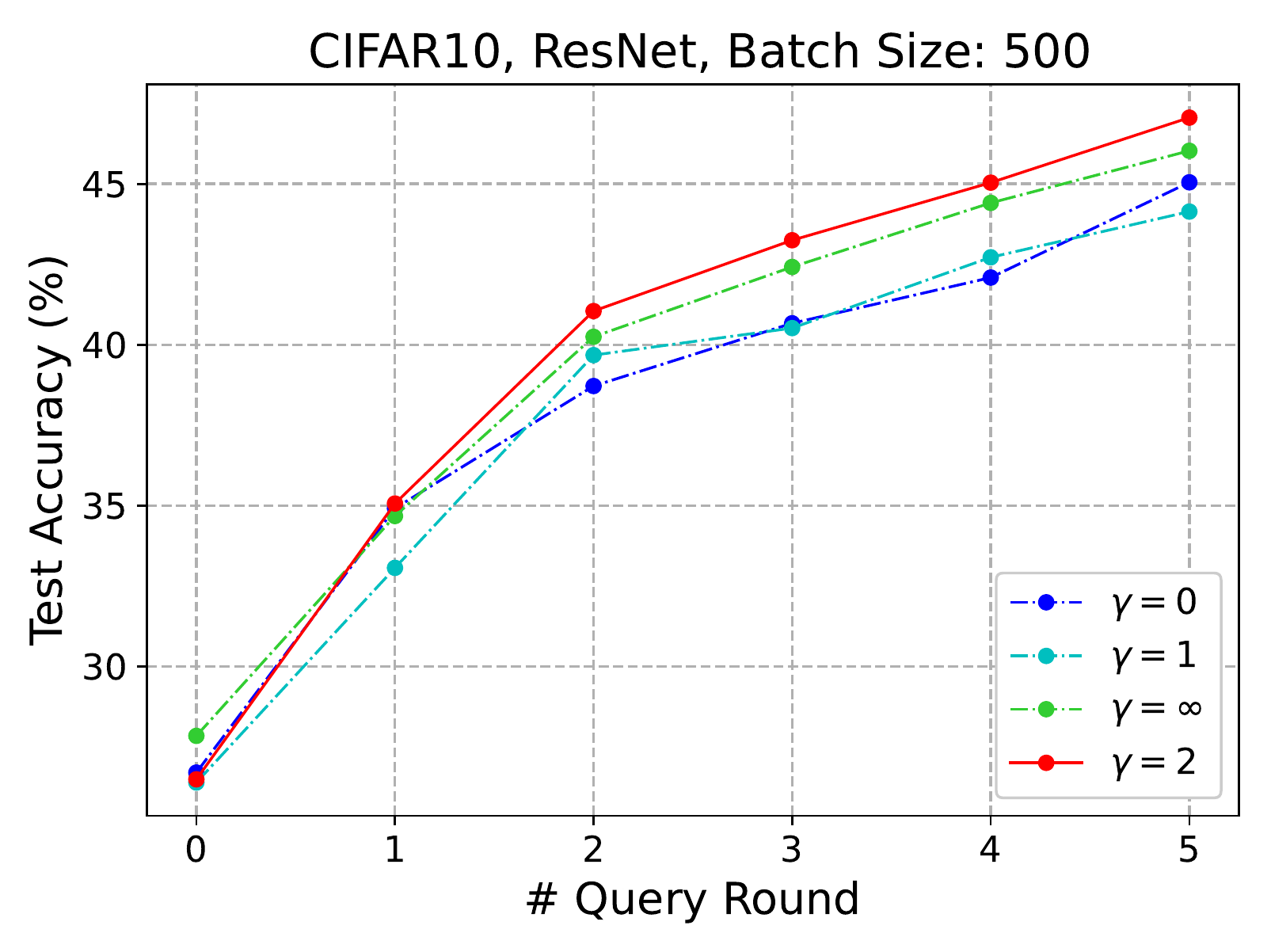}
\end{minipage}
\vspace{-4pt}
\caption{{Test Accuracy of different variants.}}
\label{fig:variants}
\vspace{-7pt}
\end{figure}

\section{Related work}
\label{sec:Related Work}
\vspace{-7 pt}
\textbf{Neural Tangent Kernel (NTK):}
Recent study has shown that under proper conditions, an infinite-width neural network can be simplified as a linear model with Neural Tangent Kernel (NTK)~\cite{jacot2018NTK}.
Since then, NTK has become a powerful theoretical tool to analyze the behavior of deep learning architecture (CNN, GNN, RNN)~\cite{arora2019exact, du2019GNTK, alemohammad2020recurrent}, random initialization~\cite{huang2020neural}, stochastic neural network~\cite{huang2022demystify}, and graph neural network \cite{huang2021towards} from its output dynamics and to characterize the convergence and generalization error~\cite{arora2019fineGrained}.
Besides, \cite{hanin2019finiteNTK} studies the finite-width NTK, aiming at making the NTK more practical.

\textbf{Active Learning:}
Active learning aims at interactively query labels for unlabeled data points to maximize model performances \cite{settles2009active}. Among others, there are two popular strategies for active learning, \ie, diversity sampling~\cite{du2015exploring, volpi2018generalizing, zhdanov2019diverse} and uncertainty sampling~\cite{roy2001toward, zhu2012uncertainty, yang2016active, ash2019deep, yoo2019learning, settles2007multiple, liu2021influence}. Recently, several papers proposed to use gradient to measure uncertainty~\cite{settles2007multiple, ash2019deep, liu2021influence}. However, those methods need to compute gradient for each class, and thus they can hardly be applied on data sets with a large class number. Besides, recent works~\citep{awasthi2021neural, Ban2022ImprovedAF} leverage NTK to analyze contextual bandit with streaming data, which are hard to be applied into our pool-based setting.
\section{Conclusion}
\label{summary}
In this work, we bridge the gap between the theoretic findings of deep neural networks and real-world deep active learning applications.
By exploring the connection between the generalization performance and the training dynamics, we propose a theory-driven method, \dgal, which selects samples to maximize training dynamics. We prove that the convergence speed of training and the generalization performance is (positively) strongly correlated under the ultra-wide condition and we show that maximizing the training dynamics will lead to a lower generalization error. 
Empirically, our work shows that \dgal\ not only consistently outperforms strong baselines across various setting, but also scales well on large deep learning models.

\section{Acknowledgment}
This work is supported by National Science Foundation (IIS-1947203, IIS-2117902, IIS-2137468, and IIS-2134079, and CNS-2125626), and by a joint ACES-ICGA funding initiative via USDA Hatch ILLU-802-946, and Agriculture and Food Research Initiative (AFRI) grant no. 2020-67021-32799/project accession no.1024178 from the USDA National Institute of Food and Agriculture. The views and conclusions are those of the authors and should not be interpreted as representing the official policies of the funding agencies or the government.


\bibliographystyle{unsrt}
\bibliography{egbib}


\section*{Checklist}


\begin{enumerate}

\item For all authors...
\begin{enumerate}
  \item Do the main claims made in the abstract and introduction accurately reflect the paper's contributions and scope?
    \answerYes{}
  \item Did you describe the limitations of your work?
    \answerYes{See appendix.}
  \item Did you discuss any potential negative societal impacts of your work?
    \answerYes{See appendix.}
  \item Have you read the ethics review guidelines and ensured that your paper conforms to them?
    \answerYes{}
\end{enumerate}

\item If you are including theoretical results...
\begin{enumerate}
  \item Did you state the full set of assumptions of all theoretical results?
    \answerYes{See appendix.}
        \item Did you include complete proofs of all theoretical results?
    \answerYes{See appendix.}
\end{enumerate}

\item If you ran experiments...
\begin{enumerate}
  \item Did you include the code, data, and instructions needed to reproduce the main experimental results (either in the supplemental material or as a URL)?
    \answerYes{See supplemental materials.}
  \item Did you specify all the training details (e.g., data splits, hyperparameters, how they were chosen)?
    \answerYes{See Section~\ref{exp} and appendix.}
        \item Did you report error bars (e.g., with respect to the random seed after running experiments multiple times)?
    \answerYes{We report the standard deviation for all experiments. See appendix.}
        \item Did you include the total amount of compute and the type of resources used (e.g., type of GPUs, internal cluster, or cloud provider)?
    \answerYes{See appendix.}
\end{enumerate}

\item If you are using existing assets (e.g., code, data, models) or curating/releasing new assets...
\begin{enumerate}
  \item If your work uses existing assets, did you cite the creators?
   \answerYes{See appendix.}
  \item Did you mention the license of the assets?
    \answerYes{}
  \item Did you include any new assets either in the supplemental material or as a URL?
    \answerYes{}
  \item Did you discuss whether and how consent was obtained from people whose data you're using/curating?
   \answerYes{}
  \item Did you discuss whether the data you are using/curating contains personally identifiable information or offensive content?
    \answerYes{}
\end{enumerate}

\item If you used crowdsourcing or conducted research with human subjects...
\begin{enumerate}
  \item Did you include the full text of instructions given to participants and screenshots, if applicable?
     \answerNA{}
  \item Did you describe any potential participant risks, with links to Institutional Review Board (IRB) approvals, if applicable?
     \answerNA{}
  \item Did you include the estimated hourly wage paid to participants and the total amount spent on participant compensation?
     \answerNA{}
\end{enumerate}

\end{enumerate}


\appendix
\setcounter{prop}{0}
\setcounter{thm}{0}
\setcounter{lem}{0}
\setcounter{cor}{0}
\setcounter{rem}{0}

\newpage
\section{APPENDIX: Derivation of Objectives}
\label{dev_obj}
For the notational convenience, we use $f(x)$ to represent $f(x; \theta)$ in the Appendix.
\subsection{Training Dynamics for Cross-Entropy Loss}
\label{deriv_train}
The partial derivative for softmax function can be defined with the following, 
\begin{equation}
\frac{\partial \sigma^i(f(x))}{\partial f^j(x)} = 
\begin{cases}
\sigma^i(f(x)) \big( 1 - \sigma^i(f(x)) \big), &  i = j, \\
-\sigma^i(f(x)) \sigma^j(f(x)), & i \neq j
\end{cases}
\end{equation}
Then, we have:

\begin{equation}
\begin{aligned}
& \frac{\partial \ell (f(x), y)}{\partial t} = - \sum_i y^i \frac{\partial \log \sigma^if(x) }{\partial \sigma^i(f(x))} \frac{ \partial \sigma^i(f(x)) }{ \partial t } \\
&= -\sum_i y^i \frac{1}{\sigma^i(f(x))} \sum_j \frac{\partial \sigma^i(f(x))}{\partial f^j(x)} \frac{\partial f^j(x)}{\partial t} \\
& = -\sum_i y^i \sum_j \big( \mathbbm{1}[i==j] - \sigma^j(f(x)) \big) \frac{\partial f^j(x)}{\partial t} \\
& = -\sum_i \big( y^i - \sigma^i(f(x)) \big) \nabla_{\theta}{f^i(x)} \nabla_{t} \theta
\end{aligned}
\end{equation}

\subsection{Derivation for Cross-Entropy Loss}
\label{deriv_CE}
\begin{equation}
\begin{aligned}
& \frac{\partial \ell (f(x), y)}{\partial \theta} = \frac{\partial \ell}{\partial f(x)} \frac{\partial f(x)}{\partial \theta}  = -\sum_{i} y^i \frac{1}{\sigma^i(f(x))} \frac{\partial \sigma^i(f(x))}{f(x)} \frac{\partial f(x)}{\partial \theta} \\
& = -\sum_{i} y^i \frac{1}{\sigma^i(f(x))}\sigma^i(f(x)) \sum_{j} \big( \mathbbm{1}[i==j] - \sigma^j(f(x)) \big)
\frac{\partial f^j(x)}{\partial \theta} \\
& =  \sum_{j} \big(\sigma^j(f(x)) - y^j \big)
\frac{\partial f^j(x)}{\partial \theta} 
\end{aligned}
\end{equation}

\subsection{APPENDIX: Training Dynamics for Mean Squared Error}
\label{mse_gradient}
For the labeled data set $S$, we define the Mean Squared Error(MSE) as:
$$
L_{MSE}(S) = \sum_{(x,y) \in S} \ell_{MSE}(f(x), y) = -\sum_{(x,y) \in S} \sum_{i\in[K]} \frac{1}{2}(f^i(x) - y^i)^2 
$$

Then the training loss dynamics for each sample can be defined as:
$$
\frac{\partial \ell_{MSE} (f(x), y)}{\partial t} = -\sum_i \big( y^i - f^i(x) \big) \nabla_{\theta}{f^i(x)} \nabla_{t} \theta 
$$

Because neural networks are optimized by gradient descent, thus:

$$
 \nabla_{t} \theta  = \theta_{t+1} - \theta_{t} = \sum_{ (x,y) \in S }\frac{\partial \ell (f(x), y)}{\partial \theta} = \sum_{(x,y) \in S}\sum_{j} \big(f^j(x) - y^j \big) \frac{\partial f^j(x)}{\partial \theta} 
$$

Therefore, the training dynamics of MSE loss can be expressed as:

$$
G_{MSE}(S) = -\frac{1}{\eta} \frac{\partial \sum_{(x,y) \in S} \ell_{MSE}(f(x), y) }{\partial t} = (f(X)-Y)^{\top}\mathcal{K}({X},{X}) (f({X})-Y)
$$

\subsection{APPENDIX: Decomposition of the Change of Training Dynamics}
\label{training_dynamics_decomposition}

According to the definition of training dynamics ( Equation~\eqref{equ:gradient} ),  we have,
\begin{equation*}
\begin{small}
\begin{aligned}
G(S) = \sum_{i,j} \sum_{ ({x}_{l}, {y}_{l}) \in S } \big( \sigma^i(f(x_l; \theta)) - y_l^{i} \big)  \sum_{ ({x}_{l^{'}}, {y}_{l^{'}}) \in S } \nabla_{\theta} f^{i}(x_l; \theta)^{\top} \nabla_{\theta} f^{j}(x_{l^{'}}; \theta) \big(\sigma^j(f(x_{l^{'}}; \theta)) - y_{l^{'}}^j\big)
\end{aligned}
\end{small}
\end{equation*}
\begin{equation*}
\begin{small}
\begin{aligned}
G(S\cup \widehat{Q}) = \sum_{i,j} \sum_{ ({x}, {y}) \in S\cup \widehat{Q} }  & \big( \sigma^i(f(x; \theta)) - y^{i} \big) \sum_{ ({x}', {y}') \in S\cup \widehat{Q} } \nabla_{\theta} f^{i}(x; \theta)^{\top} \nabla_{\theta} f^{j}(x'; \theta) \big(\sigma^j(f(x'; \theta)) - y'^j\big)
\end{aligned}
\end{small}
\end{equation*}
The change of training dynamics, $\Delta(\widehat{Q}|S) = G(S \cup \widehat{Q}) - G(S)$, can be further simplified as:

\begin{equation*}
\begin{small}
\begin{aligned}
&\Delta(\widehat{Q}|S) = G(S \cup \widehat{Q}) - G(S)\\
&=2\sum_{i,j} \sum_{ ({x}_{u}, \widehat{y}_{u}) \in \widehat{Q} }  \big( \sigma^i(f(x_u; \theta)) - \widehat{y}_u^{i} \big)  \sum_{ ({x}_{l}, {y}_{l}) \in S } \nabla_{\theta} f^{i}(x_u; \theta)^{\top} \nabla_{\theta} f^{j}(x_{l}; \theta) \big(\sigma^j(f(x_{l}; \theta)) - y_{l}^j\big) \\
&+ \sum_{i,j} \sum_{ ({x}_{u}, \widehat{y}_{u}) \in \widehat{Q} }  \big( \sigma^i(f(x_u; \theta)) - \widehat{y}_u^{i} \big)  \nabla_{\theta} f^{i}(x_u; \theta)^{\top} \nabla_{\theta} f^{j}(x_{u}; \theta) \big(\sigma^j(f(x_{u}; \theta)) - \widehat{y}_{u}^j\big) \\
& +\sum_{i,j} \sum_{ ({x}_{u}, \widehat{y}_{u}) \in \widehat{Q} }  \big( \sigma^i(f(x_u; \theta)) - \widehat{y}_u^{i} \big)  \sum_{ ({x}_{u'}, \widehat{y}_{u'}) \in \widehat{Q}, u' \neq u} \nabla_{\theta} f^{i}(x_{u'}; \theta)^{\top} \nabla_{\theta} f^{j}(x_{u'}; \theta) \big(\sigma^j(f(x_{u'}; \theta)) - \widehat{y}_{u'}^j\big) \\
& = \sum_{(x_u, \widehat{y}_u) \in \widehat{Q}}  \Delta(\{(x_u, \widehat{y}_u) \}|S) +  \sum_{(x_u, \widehat{y}_u), (x_{u'}, \widehat{y}_{u'}) \in \widehat{Q}} {d^i}(x_u, \widehat{y}_u)^{\top} \mathcal{K}^{ij}(x_u,x_{u'}) {d^i}(x_{u'}, \widehat{y}_{u'})
\end{aligned}
\end{small}
\end{equation*}

\subsection{APPENDIX: Simplification of the Change of Training Dynamics}
\label{training_dynamics_simplification}

\begin{equation*}
\begin{small}
\begin{aligned}
 \Delta(\{(x_u, \widehat{y}_u) \}|S) =  & 2\sum_{i,j} \sum_{ ({x}_{u}, \widehat{y}_{u}) \in \widehat{Q} }  \big( \sigma^i(f(x_u; \theta)) - \widehat{y}_u^{i} \big)  \sum_{ ({x}_{l}, {y}_{l}) \in S } \nabla_{\theta} f^{i}(x_u; \theta)^{\top} \nabla_{\theta} f^{j}(x_{l}; \theta) \big(\sigma^j(f(x_{l}; \theta)) - y_{l}^j\big) \\
&+\sum_{i,j} \sum_{ ({x}_{u}, \widehat{y}_{u}) \in \widehat{Q} }  \big( \sigma^i(f(x_u; \theta)) - \widehat{y}_u^{i} \big)  \nabla_{\theta} f^{i}(x_u; \theta)^{\top} \nabla_{\theta} f^{j}(x_{u}; \theta) \big(\sigma^j(f(x_{u}; \theta)) - \widehat{y}_{u}^j\big)
\end{aligned}
\end{small}
\end{equation*}

The derivative of loss with respect to model parameters can be written as:
\begin{equation*}
\begin{small}
\begin{aligned}
\frac{\partial \sum_{(x,y) \in S} \ell (f({x}; \theta), {y})}{\partial \theta} 
= \sum_{(x,y) \in S} \sum_{j \in [K]} \big(\sigma^j(f({x}; \theta)) - {y}^j \big) \nabla_{\theta} f^{j}(x; \theta)
\end{aligned}
\end{small}
\end{equation*}

Therefore, the change of training dynamics caused by $\{(x_u, \widehat{y}_u) \}$ can be written as:

\begin{equation*}
\begin{small}
\begin{aligned}
  \Delta(\{(x_u, \widehat{y}_u) \}|S) = 
  \|\nabla_{\theta} \ell (f({x}_u; \theta), \hat{{y}}_u)\|^2 + 2 \sum_{(x, y) \in S}  \nabla_{\theta}\ell(f({x}_u; \theta), \hat{{y}}_u)^{\top} \nabla_{\theta} \ell (f({x}; \theta), {{y}})
\end{aligned}
\end{small}
\end{equation*}

\color{black}

\section{APPENDIX: Proofs for Theoretical Analysis}
\label{proofs}

\subsection{Proofs for Theorem 1}
\label{proof:them1}
\begin{lem} [Convergence Analysis with NTK, Theorem 4.1 of \cite{arora2019fineGrained}] 
Suppose $\lambda_0 = \lambda_{\min}(\boldsymbol{\Theta}) > 0$ for all subsets of data samples. For $\delta \in (0,1)$, if $m = \Omega(\frac{n^7}{\lambda^4_0  \delta^4 \epsilon^2})$ and $\eta = O(\frac{\lambda_0}{n^2})$, with probability at least $1-\delta$, the network can achieve near-zero training error,
\begin{small}
\begin{equation} 
\norm{{Y}-f_t({X; \theta(t)})}_2 = \sqrt{\sum_{k=1}^K \sum_{i=1}^{n} (1-\eta \lambda_i)^{2t} (\Vec{v}^{\top}_i {Y}^k)^2} \pm \epsilon 
\end{equation}
\end{small}

where $n$ denotes the number of training samples and $m$ denotes the width of hidden layers. The NTK $\boldsymbol{\Theta} = V^{\top} \Lambda V$ with $\Lambda = \{\lambda_i \}_{i=1}^n$ is a diagonal matrix of eigenvalues and $V =\{\Vec{v}_i \}_{i=1}^n$ is a unitary matrix.
\end{lem}

\begin{proof}
According to \cite{arora2019fineGrained}, if $m = \Omega(\frac{n^7}{\lambda^4_0 \delta^4 \epsilon^2})$ and learning ratio $\eta = O(\frac{\lambda_0}{n^2})$, then with probability at least $1-\delta$ over the random initialization, we have, $
\norm{Y_l-f_t(X; \theta(t))}_2 = \sqrt{\sum_{k=1}^K\sum_{i=1}^n (1-\eta \lambda_i)^{2t} (v^{\top}_i Y_l^k)^2} \pm \epsilon
$. We decompose the NTK using $\boldsymbol{\Theta} = V^{\top} \Lambda V$ with $\Lambda = \{\lambda_i \}_{i=1}^n$ a diagonal matrix of eigenvalues and $V =\{v_i \}_{i=1}^n$ a unitary matrix.  
At each training step in active learning, the labeled samples will be updated by $S = S \cup \overline{Q} $. We can apply the convergence result in each of this step and achieve near zero error.
\end{proof}

\begin{thm} [Relationship between convergence rate and alignment]
Under the same assumptions as in Lemma \ref{lem:convergence}, the convergence rate described by $\mathcal{E}_t$ satisfies, 
\begin{equation}
 \Tr [Y^{\top} Y] - 2t\eta \mathcal{A}({X},Y) \le \mathcal{E}^2_t(X,Y) \le \Tr[Y^{\top} Y] -\eta \mathcal{A}({X},Y) 
\end{equation}
\end{thm}

\begin{proof} We first prove the inequality on the right hand side. It is easy to see that $(1-\eta \lambda_i)^{2t} \le (1-\eta \lambda_i) $ for each $\lambda_i$ and $t \ge 1$, based on the fact that $\forall \lambda_i$, $0 \le 1-\eta \lambda_i \le 1$. Then we can obtain,
$$
\begin{aligned}
\mathcal{E}_t(X,Y) & = \sqrt{\sum_{k=1}^K \sum_{i=1}^{n} (1-\eta \lambda_i)^{2t} (v^{\top}_i Y^k)^2}
\le \sqrt{\sum_{k=1}^K \sum_{i=1}^{n} (1-\eta \lambda_i) (v^{\top}_i Y^k)^2} \\
& = \sqrt{\Tr[ Y^{\top}(I-\eta \boldsymbol{\Theta}) Y] } = \sqrt{\Tr[ Y^{\top} Y] - \eta \mathcal{A}(X,Y) }
\end{aligned}
$$

Then we use Bernoulli's inequality to prove the inequality on the left hand side.  Bernoulli's inequality states that, ${\displaystyle (1+x)^{r}\geq 1+rx}$, for every integer $r \ge 0$ and every real number $x \ge -1$. It is easy to check that $(-\eta \lambda_i) \ge -1$, $\forall \lambda_i$. Therefore,
$$
\begin{aligned}
\mathcal{E}_t(X,Y) & = \sqrt{\sum_{k=1}^K \sum_{i=1}^{n} (1-\eta \lambda_i)^{2t} (v^{\top}_i Y^k)^2}
\ge \sqrt{\sum_{k=1}^K \sum_{i=1}^{n} (1-2t \eta \lambda_i) (v^{\top}_i Y^k)^2} \\
& = \sqrt{\Tr[ Y^{\top}(I- 2t \eta \boldsymbol{\Theta}) Y] } = \sqrt{\Tr[ Y^{\top} Y ] - 2t \eta \mathcal{A}(X,Y) }
\end{aligned}
$$
\end{proof}

\subsection{Proof for Theorem 2}
\label{proof:them2}
\begin{lem} [Generalization bound with NTK, Theorem 5.1 of \cite{arora2019fineGrained}] Suppose data $S = \{ (x_i,y_i)\}_{i=1}^n$ are i.i.d. samples from a non-degenerate distribution $p(x,y)$, and $m \ge {\rm poly}(n, \lambda_0^{-1}, \delta^{-1})$. Consider any loss function $\ell: \mathbb{R} \times \mathbb{R} \rightarrow [0,1]$ that is $1$-Lipschitz, with probability at least $1-\delta$ over the random initialization, the network trained by gradient descent for $T \ge \Omega(\frac{1}{\eta \lambda_0} \log \frac{n}{\delta})$ iterations has population risk $\mathcal{L}_p = \mathbb{E}_{(x,y)\sim p(x,y)}[\ell(f_T(x),y)]$ that is bounded as follows:
\small
\begin{equation}
   \mathcal{L}_p \le \sqrt{\frac{2 \Tr[{Y^{\top} \boldsymbol{\Theta}^{-1}(X, X) Y}]}{n}}+O \bigg(\sqrt{\frac{\log \frac{n}{\lambda_0 \delta}}{n}} \bigg).
\end{equation}
\normalsize
\end{lem}

\begin{proof}
We first show that the generalization bound regrading our method on ultra-wide networks. The distance between weights of trained networks and their initialization values can be bounded as,
$\norm{w_r(t)-w_r(0)} = O(\frac{n}{\sqrt{m}\lambda_0 \sqrt{\delta}})$. We then give a bound on the $\norm{W(t)-W(0)}_F$, where $W =\{w_1, w_2, \dots \}$ is the set of all parameters. We definite $Z = \frac{\partial f(t)}{\partial W(t)}$, then the update function is given by $W(t+1) = W(t) - \eta Z (Z^{\top} W(t) -Y)$.    
Summing over all the time step $t =0,1,\dots$, we can obtain that $W(\infty)-W(0) = \sum_{t=0}^{\infty} \eta Z(I-\eta \boldsymbol{\Theta}) y = Z\boldsymbol{\Theta}^{-1}Y$. Thus the distance can be measured by $
\norm{W(\infty)-W(0)}^2_F =  \Tr[{Y^{\top} \boldsymbol{\Theta}^{-1} Y}]$.

Then the key step is to apply Rademacher complexity. Given $R>0$, with probability at least $1-\delta$, simultaneously for every $B>0$, the function class $\mathcal{F}_{B,R} = \{ f: \norm{w_r(t)-w_r(0)} \le R~ (\forall r \in m), \norm{W(\infty)-W(0)}^2_F \le B \}$  has empirical Rademacher complexity bounded as,
$$
\mathcal{R}_S(\mathcal{F}_{B,R}) = \frac{1}{n} \mathbb{E}_{\epsilon_i \in \{ \pm 1\}^n } \bigg[ \sup_{f \in \mathcal{F}_{B,R}} \sum_{i=1}^n \epsilon_i f(x_i) \bigg] \le  \frac{B}{\sqrt{2n}}\big(1+(\frac{2\log\frac{2}{\delta}}{m})^{1/4} \big) + 2R^2\sqrt{m} + R\sqrt{2\log\frac{2}{\delta}}
$$
where $ B = \sqrt{ \Tr[{Y^{\top} \boldsymbol{\Theta}^{-1}(X,X) Y}]} $, and $R =\frac{n}{\sqrt{m}\lambda_0 \sqrt{\delta}} $. 

Finally, Rademacher complexity directly gives an upper bound on generalization error \citep{mohri2018foundations},
$
\sup_{f \in \mathcal{F}} \{ \mathcal{L}_p(f)- \mathcal{L}_S(f) \} \le 2 \mathcal{R}_{S} + 3c \sqrt{\frac{\log(2/\delta)}{2n}}
$, where $\mathcal{L}_S(f) \le \frac{1}{\sqrt{n}}$. Based on this, we apply a union bound over a finite set of different $i$'s. Then with probability at least $1-\delta/3$ over the sample $S$, we have $\sup_{f \in \mathcal{F}_{R, B_i}} \{ \mathcal{L}_p (f) - \mathcal{L}_S (f) \} \le  2\mathcal{R}_S(\mathcal{F}_{B_i, R}) + O(\sqrt{\frac{\log \frac{n}{\lambda_0\delta}}{n}}),~\forall i \in  \{1,2,\ldots, O(\frac{n}{\lambda_0}) \}$. Taking a union bound, we know that with probability at least $1-\frac{2}{3} \delta$ over the sample $S$, we have, $f_T \in \mathcal{F}_{B^*_i,R}$ for some $i^\ast$, $\mathcal{R}_S (\mathcal{F}_{B^*_i,R}) \le \sqrt{\frac{\Tr[{Y^{\top} \boldsymbol{\Theta}^{-1}(X,X) Y}]}{2n}} + \frac{2}{\sqrt{n}}$ and $\sup_{f_T \in \mathcal{F}_{B^*_i,R}}\{\mathcal{L}_p(f_T) - \mathcal{L}_S(f_T)\} \le 2\mathcal{R}_S(\mathcal{F}_{B^*_i,R})+O(\sqrt{\frac{\log \frac{n}{\lambda_0 \delta}}{n}})$. These together can imply,
$$
\mathcal{L}_p(f) \le \frac{1}{\sqrt{n}} + 2 \mathcal{R}_{S}(\mathcal{F}_{B^*_i,R}) + O(\sqrt{\frac{\log \frac{n}{\lambda_0 \delta}}{n}}) \le \sqrt{\frac{2\Tr[{Y^{\top} \boldsymbol{\Theta}^{-1}(X,X) Y}]}{n}} + O\bigg(\sqrt{\frac{\log \frac{n}{\lambda_0 \delta}}{n}}\bigg).
$$
More proof details can be found in \cite{arora2019fineGrained}.
\end{proof}

\begin{thm} [Relationship between the generalization bound and alignment]
Under the same assumptions as in Lemma \eqref{lem:generalization}, if we define the generalization upper bound as $\mathcal{B}(X,Y) = \sqrt{\frac{2 \Tr[{Y^{\top} \boldsymbol{\Theta}^{-1} Y}]}{n}}$, then it can be bounded with the alignment as follows,
\begin{equation}
\frac{ \Tr^2[Y^{\top} Y]}{ \mathcal{A}({X},Y)} \le \frac{n}{2}\mathcal{B}^2({X},Y) \le \frac{\lambda_{max}}{\lambda_{min}} \frac{  \Tr^2 [Y^{\top} Y] }{\mathcal{A}({X},Y) } 
\end{equation}               
\end{thm}

\begin{proof}
We first expand the following expression:
$$
\frac{n}{2} \mathcal{B}^2(X,Y)  \mathcal{A}(X,Y) =
\sum_{k=1}^K \sum_{i=1}^n \lambda_i (v^{\top}_i Y^k)^2 \sum_{k=1}^K \sum_{i=1}^n \frac{1}{\lambda_i} (v^{\top}_i Y^k)^2
$$
Then we use this expansion to prove the inequality on the left hand side,

$$
\begin{aligned}
 &\sum_{k=1}^K \sum_{i=1}^n \lambda_i (v^{\top}_i Y^k)^2 \sum_{k=1}^K \sum_{i=1}^n \frac{1}{\lambda_i} (v^{\top}_i Y^k)^2  = \sum_{k=1}^K \sum_{k'=1}^K \bigg( \sum_{i=1}^n \lambda_i (v^{\top}_i Y^k)^2 \sum_{i=1}^n \frac{1}{\lambda_i} (v^{\top}_i Y^{k'})^2 \bigg) \\
& \ge  \sum_{k=1}^K \sum_{k'=1}^K
  \bigg( \sum_{i=1}^n (v^{\top}_i Y^k)^2 \sum_{i=1}^n  (v^{\top}_i Y^{k'})^2 \bigg) = \big( \sum_{k=1}^K {Y^k}^{\top} V^{\top} V Y^k \big) \big( \sum_{k=1}^K {Y^k}^{\top} V^{\top} V Y^k \big) \\
& = \Tr^2 [Y^{\top} Y]
\end{aligned} 
$$
The second line is due to quadratic mean is greater or equal to geometric mean. Finally, we prove the inequality on the right hand side,
$$
\begin{aligned}
 &\sum_{k=1}^K \sum_{i=1}^n \lambda_i (v^{\top}_i Y^k)^2 \sum_{k=1}^K \sum_{i=1}^n \frac{1}{\lambda_i} (v^{\top}_i Y^k)^2  = \sum_{k=1}^K \sum_{k'=1}^K \bigg( \sum_{i=1}^n \lambda_i (v^{\top}_i Y^k)^2 \sum_{i=1}^n \frac{1}{\lambda_i} (v^{\top}_i Y^{k'})^2 \bigg) \\
& \le  \sum_{k=1}^K \sum_{k'=1}^K \frac{\lambda_{max}}{\lambda_{min}}
  \bigg( \sum_{i=1}^n (v^{\top}_i Y^k)^2 \sum_{i=1}^n  (v^{\top}_i Y^{k'})^2 \bigg) = \frac{\lambda_{max}}{\lambda_{min}} \big( \sum_{k=1}^K {Y^k}^{\top} V^{\top} V Y^k \big) \big( \sum_{k=1}^K {Y^k}^{\top} V^{\top} V Y^k \big) \\
& = \frac{\lambda_{max}}{\lambda_{min}} \Tr^2 [Y^{\top} Y]
\end{aligned} 
$$
\end{proof}

\subsection{Derivation for Maximum Mean Discrepancy}
\label{derivation_mmd}
The difference between truth risk over $p(x)$ and $q(x)$ can be defined as,
\begin{equation*}
\begin{small}
\begin{aligned}
\mathcal{L}_{{p}} - \mathcal{L}_{{q}} = \int_{{x}} g({x}) p({x}) d {x}-\int_{{x}} g({x}) q({x}) d {x}
\end{aligned}
\end{small}
\end{equation*}
where $g(x) = \int_{y} \ell(f({x; \theta}), y) p(y|x) dy$. Follow~\cite{wang2015querying}, we assume that the prediction functions have bounded norm $\|f\|_{F}$. Thus, the function $g$ is bounded. By given the loss function, $g$ is also measurable. Then, $\exists ~ \hat{g} \in \mathcal{C}(x)$, such that, 

\begin{equation*}
\begin{small}
\begin{aligned}
& \int_{{x}} g({x}) p({x}) d {x}-\int_{{x}} g({x}) q({x}) d {x}
= \int_{{x}} \hat{g}({x}) p({x}) d {x}-\int_{{x}} \hat{g}({x}) q({x}) d {x} \\
 & \leq  \sup _{\hat{g} \in \mathcal{C}({x})} \int_{{x}} \hat{g}({x}) p({x}) d {x}-\int_{{x}} \hat{g}({x}) q({x}) d {x} = \text{MMD}\big(p({x}), q({x}), \mathcal{C}\big)
\end{aligned}
\end{small}
\end{equation*}

where $\mathcal{C}(x)$ is the function class of bounded and continuous functions of $x$. To make the MMD term be measurable, we empirically restrict the MMD on a reproducing kernel Hilbert space (RKHS) with the
characteristic kernel $\mathcal{H}_{\boldsymbol{\Theta}}$. Following~\cite{gretton2012kernel}, we know that the relationship between the true MMD and the empirical MMD is, 

\begin{equation*}
\begin{aligned}
{P}\Big(\big| \text{MMD}\big(p({x}), q({x}), \mathcal{C}\big) - \text{MMD}( S_0, S, \mathcal{H}_{\boldsymbol{\Theta}} )\big| \geq \epsilon+2(\sqrt{\frac{C}{n_0}}+\sqrt{\frac{C}{n}})\Big) 
\quad \leq 2 e^{\frac{-\epsilon^{2} n_0 n}{2 C(n_0+n)}}
\end{aligned}
\end{equation*}
where $\text{MMD}( S_0, S, \mathcal{H}_{\boldsymbol{\Theta}})$ is the empirical measure for $\text{MMD}( p(x), q(x), \mathcal{H}_{\boldsymbol{\Theta}})$. Slightly overloading the notation, we denote $S \sim q(x)$, which may not be i.i.d., and the initial label set $S_0 \sim p(x)$. Then, in the active learning setting, $S_0 \subseteq S$.  Further, we denote $|S_0|=n_0, |S|=n$ and $\forall x,x' \in S, \boldsymbol{\Theta}(x,x') \leq C$. Therefore, we have, $\sqrt{\frac{C}{n}}+\sqrt{\frac{C}{n_0}} \geq 2 \sqrt{\frac{C}{n}}$.
For constant factor $\gamma = \frac{M}{M+B}$, we have the following inequality, 
\begin{equation*}
\begin{aligned}
{P}\big(\text{MMD}\big(p({x}), q({x}), \mathcal{C}\big) \geq  \text{MMD}( S_0, S, \mathcal{H}_{\boldsymbol{\Theta}} ) +\epsilon+ 4 \sqrt{\frac{C}{n}}\big) \leq 2 e^{\frac{-\gamma \epsilon^{2} n}{4 C}}
\end{aligned}
\end{equation*}
Denoting $2 e^{\frac{- \gamma \epsilon^{2} n}{4 C}}=\delta / 2$, then we have $\epsilon=\sqrt{\frac{4 C \ln (4 / \delta)}{\gamma n}}$. Combining all the above results, we show that with probability at least $1-\delta$, the following inequality holds:
\begin{equation*}
\begin{aligned}
\mathcal{L}_{{p}} - \mathcal{L}_{{q}} \leq \text{MMD}( S_0, S, \mathcal{H}_{\boldsymbol{\Theta}} ) + 4 \sqrt{\frac{C}{n}} + \sqrt{\frac{4 C \ln (4 / \delta)}{\gamma n}}
\end{aligned}
\end{equation*}
Then, we can get, 
\begin{equation*}
\begin{aligned}
\mathcal{L}_{{p}} - \mathcal{L}_{{q}} \leq \text{MMD}( S_0, S, \mathcal{H}_{\boldsymbol{\Theta}} )  + O\left(\sqrt{\frac{C \ln (1 / \delta)}{n}}\right)
\end{aligned}
\end{equation*}
\color{black}

\section{APPENDIX: More details of experimental settings}
\label{appendix:exp-setting}



\subsection{Implementation Detail}
\label{model_archi}
For simple CNN model, we utilize the same architecture used in Pytorch CIFAR10 Image Classification Tutorial \footnote{\url{https://pytorch.org/tutorials/beginner/blitz/cifar10_tutorial.html}}.
For ResNet model, we use the Pytorch Offical implementation of ResNet-18 \footnote{\url{https://github.com/pytorch/vision/blob/main/torchvision/models/resnet.py}} and set the output dimension to the number of classes.
For VGG model, we use the Pytorch Offical implementation of VGG-11 \footnote{\url{https://github.com/pytorch/vision/blob/main/torchvision/models/vgg.py}}. Besides, we leverage the library BackPACK~\citep{dangel2020backpack} to collect the gradient of samples in batch.

We keep a constant learning rate  of 0.001 for all three datasets and all three models. All the codes mentioned above use the MIT license. All experiments are done with four Tesla V100 SXM2 GPUs and a 12-core 2.2GHz CPU. 

\color{black}
\subsection{Computation of Acquisition Function}

The acquisition function employed by \dgal~can be written as the Equation~\ref{equ:discussion}. Furthermore, we simplify it into the following form:

\begin{equation}
\begin{small}
\begin{aligned}
  \Delta(\{(x_u,& \widehat{y}_u) \}| S) =  
  \|\nabla_{\theta} \ell (f({x}_u; \theta), \hat{{y}}_u)\|^2 + 2 \nabla_{\theta}\ell(f({x}_u; \theta), \hat{{y}}_u)^{\top} \nabla_{\theta} \ell (f({X_S}; \theta), {{Y_S}}).
\end{aligned}
\end{small}\label{equ:acquisition_apd}
\end{equation}
where $\nabla_{\theta} \ell (f({X_S}; \theta), {{Y_S}}) = \sum_{(x, y) \in S} \nabla_{\theta} \ell (f({x}; \theta), {{y}})$.
The computational requirement of the Equation~\ref{equ:acquisition_apd} is mainly composed of two parts, the computation of gradient and the computation of the inner product.
While PyTorch~\cite{paszke2019pytorch} can compute efficiently batch gradients, BackPACK ~\cite{dangel2020backpack} optimizes the computation of individual gradient and compute the gradient norm, sample per sample, at almost no time overhead. Thus, the acquisition function can be computed at low computational costs.
Note, the efficiency of BackPACK has been verified by several recent works with extensive experiments\cite{rame2022fishr, subramani2021enabling}.
\color{black}

\section{APPENDIX: Verification Experiments under Ultra-wide Condition}
\label{ultra-exp}


\subsection{Experiment Setting and Computational Detail for the Empirical Comparison between NTK and MMD}
\label{ntk_mmd}

\textbf{Experiment Setting} For the verification experiment shown in Figure \ref{fig:MMD_B}, we employ a simple CNN as the target model, in which there are three convolutional layers following with global average pooling layer, on the CFAIR10 data set. Note, this CNN architecture is widely used in NTK analysis works~\citep{arora2019exact, zandieh2021scaling}. To make the verification experiment close to the application setting, we keep size of initial labeled set and query batch size same as what we used in Section~\ref{exp}. 

\textbf{Computational Detail}
We follow ~\citep{jia2021efficient} to compute the MMD with NTK kernel. The MMD term, $\text{MMD}(p(x), q(x), \mathcal{H}_{\boldsymbol{\Theta}})$, can be simplified into the following form:

\begin{small}
\begin{equation}
\begin{aligned}
\text{MMD}^2(p(x), q(x)) = \mathbbm{E}[\boldsymbol{\Theta}(x_i,x_j) +\boldsymbol{\Theta}(x'_i,x'_j) - 2\boldsymbol{\Theta}(x_i,x'_j)]
\end{aligned}
\end{equation}
\end{small}
where $x_i, x_j \sim p(x)$ and $x'_i,x'_j \sim q(x)$. Then, we define set $S_0$ as $\{x_1, ..., x_{n_0}\} \sim p(x)$ and 
{ set $S$} 
as $\{x'_1, ..., x'_n\} \sim q(x)$, where $n_0 \leq n$. The  ${\text{MMD}}^2(S_0, {S})$ is an unbiased estimation for $\text{MMD}^2(p(x), q(x))$, 
can we explicitly computed by:
\begin{small}
\begin{equation}
\begin{aligned}
& {\text{MMD}}^2(S_0, {S})  = \frac{1}{m^{2}-m} a+\frac{1}{n^{2}-n} b-\frac{2}{m(n-1)} c \\
&a=\left( \sum_{i,j}^m \boldsymbol{\Theta}(x_i,x_j) - \sum_{i}^m \boldsymbol{\Theta}(x_i,x_i) \right) \\
&b=\left( \sum_{i,j}^n \boldsymbol{\Theta}(x'_i,x'_j) - \sum_{i}^n \boldsymbol{\Theta}(x'_i,x'_i) \right) \\
&c = \left( \sum_{i}^m \sum_{j}^n \boldsymbol{\Theta}(x_i,x'_j) - \sum_{i, j}^m \boldsymbol{\Theta}(x_i,x'_i) \right) \\
\end{aligned}
\end{equation}
\end{small}
Therefore, the MMD term of Equation~\eqref{mmd}, { $\text{MMD}(S_{0}, S, \mathcal{H}_{\boldsymbol{\Theta}})$} , can be empirically approximated by { $\sqrt{{\mathrm{MMD}}^{2}(S_{0}, S)}$}.

\subsection{Experiment for the Correlation Study between Training Dynamics and Alignment}
\label{dynamics_alignment}

\textbf{Experiment Setting.}
For the verification experiment shown in Figure~\ref{fig:align_dynamics}, we also use the simple CNN on CIFAR10. And to keep consistent with the application setting, we set $|S|=500$ and $|\overline{Q}|=250$. The $\overline{Q}$ is randomly sampled from the unlabeled set and the labeled set $S$ is fixed. We independently sample $\overline{Q}$ 150 times to compute the correlation between between $G_{MSE}(S \cup \overline{Q})$ and $\mathcal{A}(X \| X_Q, Y \| Y_Q)$.

\vspace{5pt}
\color{black}
\textbf{Correlation between Training Dynamics computed with pseudo-labels and Alignment.}
\begin{wrapfigure}{r}{5cm}
\includegraphics[width=5.1cm]{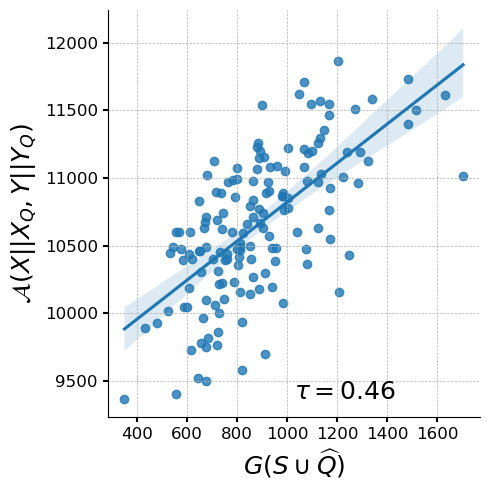}
\caption{Relation between Alignment and Training Dynamics computed with the pseudo-label.}
\label{fig:align_dynamics_pesudo_label}
\end{wrapfigure}
In Figure~\ref{fig:align_dynamics}, we compute the training dynamics with the ground-truth label. To study the effect of pseudo-labels, we further provide the relation between training dynamics computed with pseudo-labels $G_{MSE}(S\cup \overline{Q})$ and alignment $\mathcal{A}(X \| X_Q, Y \| Y_Q)$, in which we compute the pseudo-labels with $\Theta(X_Q, X)^{\top} \Theta(X, X)^{-1} Y$. The result is shown in the Figure~\ref{fig:align_dynamics_pesudo_label}. Note that we keep hyperparameters the same as previously described. 
Compared with Figure~\ref{fig:align_dynamics}, we find that the positive relationship between $\mathcal{A}$ and the $G$ computed with ground-truth labels is stronger than the $G$ computed with pseudo-labels. The result is aligned with our expectations, because the extra noise is introduced by the pseudo-labels. But, the Kendall $\tau$ coefficient still achieves $0.46$ for $\mathcal{A}$ and the $G$ computed with pseudo-labels which indicates the utility of using $G$ calculating with pseudo-labels as the acquisition function to query samples.

\subsection{Correlation Study between Training Dynamics and Generalization Bound}
We present the relation between the training dynamics and the generalization bound in Figure~\ref{fig:dynamics_bound}.
Same as the previous,  we set $|S|=500$ and $|\overline{Q}|=250$ and the $\overline{Q}$ is randomly sampled from the unlabeled set.
The result shows that with the increase of $G$, $B$ will decrease. This empirical observation is aligned with our expectation, because Theorem 2 indicates that the alignment $A$ is inverse proportional to $B$ and Figure~\ref{fig:align_dynamics} tells us that the $G$ is proportional to $A$. Besides, the $\tau$ achieves -0.253 which indicates that the $A$ is moderately inverse proportional to $B$~\cite{botsch2011chapter}.
\begin{figure}[h]
\centering
\includegraphics[width=6cm]{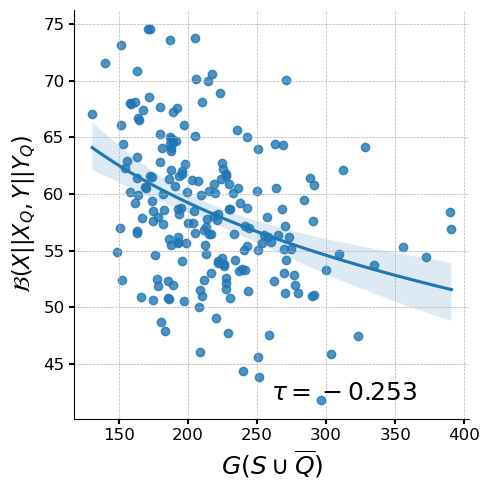}
\caption{Relationship between Training Dynamics and Generalization.}
\label{fig:dynamics_bound}
\end{figure}
\color{black}



\section{APPENDIX: More details of experimental results}

\subsection{Baselines}\label{baselines}
\begin{enumerate}
\item Random: Unlabeled data are randomly selected at each round.
\item Coreset: This method performs a clustering over the last hidden representations in the network, and calculates the minimum distance between each candidate sample’s embedding and embeddings of labeled samples. Then data samples with the maximum distances are selected. ~\citep{sener2017active}.
\item Confidence  Sampling~(Conf): The method  selects $b$ examples with smallest predicted class probability $\text{max}_{i}^{K}f^i(x; \theta)$~\citep{wang2014new}.
\item Margin Sampling~(Marg): The bottom $b$ examples sorted according to the example's multi-class margin are selected. The margin is defined as $f^i(x; \theta) - f^j(x; \theta)$, where $i$ and $j$ are the indices of the largest and second largest entries of $f(x;\theta)$ ~\citep{roth2006margin}.
\item Entropy: Top $b$ samples are selected according  to  the  entropy  of  the  example’s  predictive  class  probability  distribution,  the entropy is defined as $H((f^i(x;\theta))^K_{i=1})$, where $H(p) =\sum_{i}^K{p_i}\ln\frac{1}{p_i}$~\citep{wang2014new}.
\item Active Learning by Learning~(ALBL): The bandit-style meta-active learning algorithm combines Coreset and Conf~\citep{hsu2015active}.
\item Batch Active learning by Diverse Gradient Embeddings~(BADGE): $b$ samples are selected by using k-means++ seeding on the gradients of the final layer, in order to query by uncertainty and diversity.~\citep{ash2019deep}.
\end{enumerate}



\subsection{Experiment Results}\label{rest_main_exp}
The results for ResNet18, VGG11, and vanilla CNN on CIFAR10, SVHN, and Caltech101 with different batch sizes have been shown in the Figure~\ref{fig:main} and~\ref{fig:rest_main}. 
{Note, for the large batch size setting ($b=1000$) on Caltech101, we set the number of query round $R=4$, in which 49.2\% images will be labeled after 4 rounds.}


\begin{figure}[h!] 
\vspace{-10 pt}
\centering
\hspace{-5pt}
\begin{minipage}{0.3\textwidth}
\includegraphics[width = 1.8in]{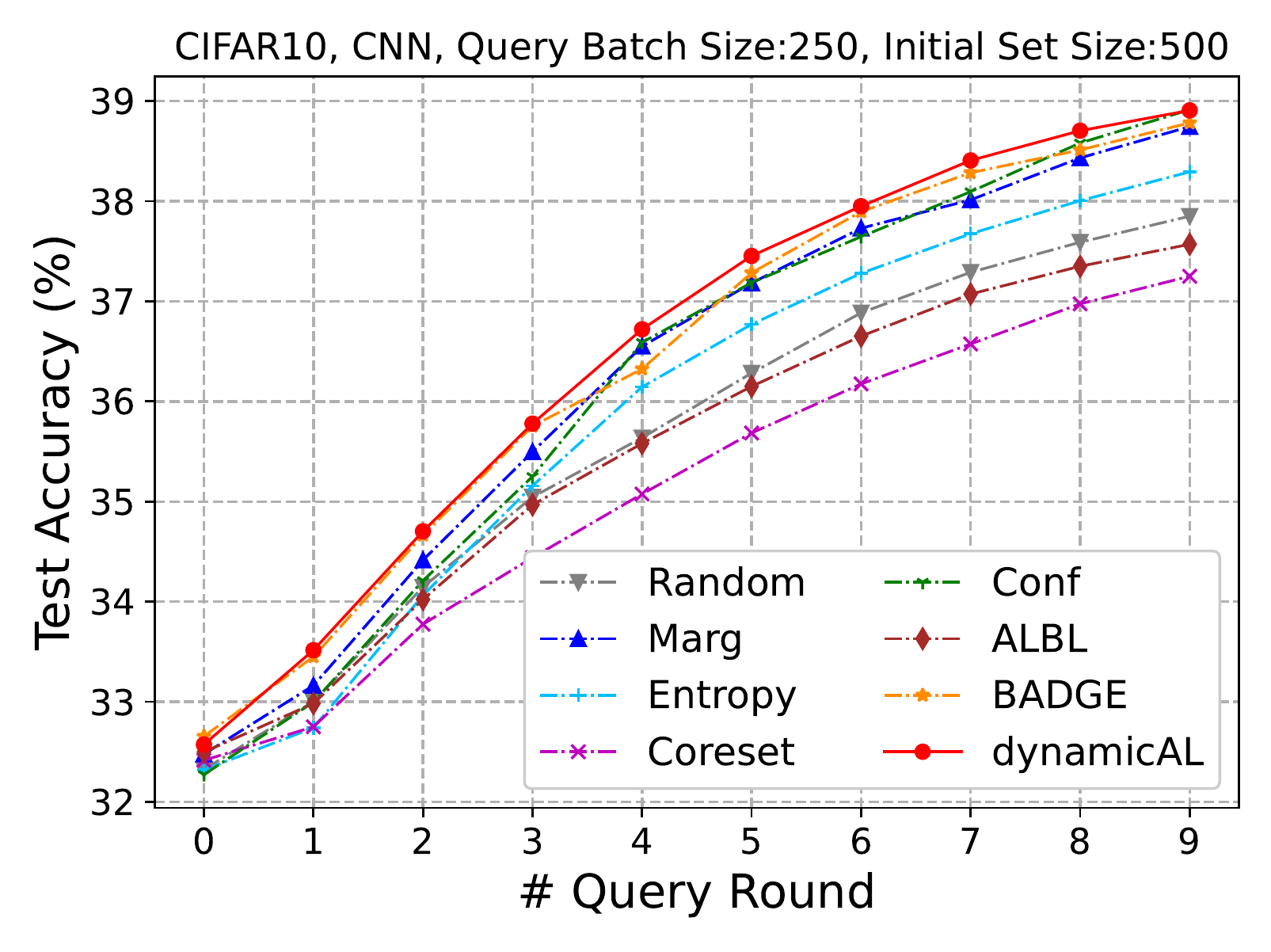}
\end{minipage}
\hspace{5pt}
\begin{minipage}{0.3\textwidth}
\includegraphics[width = 1.8in]{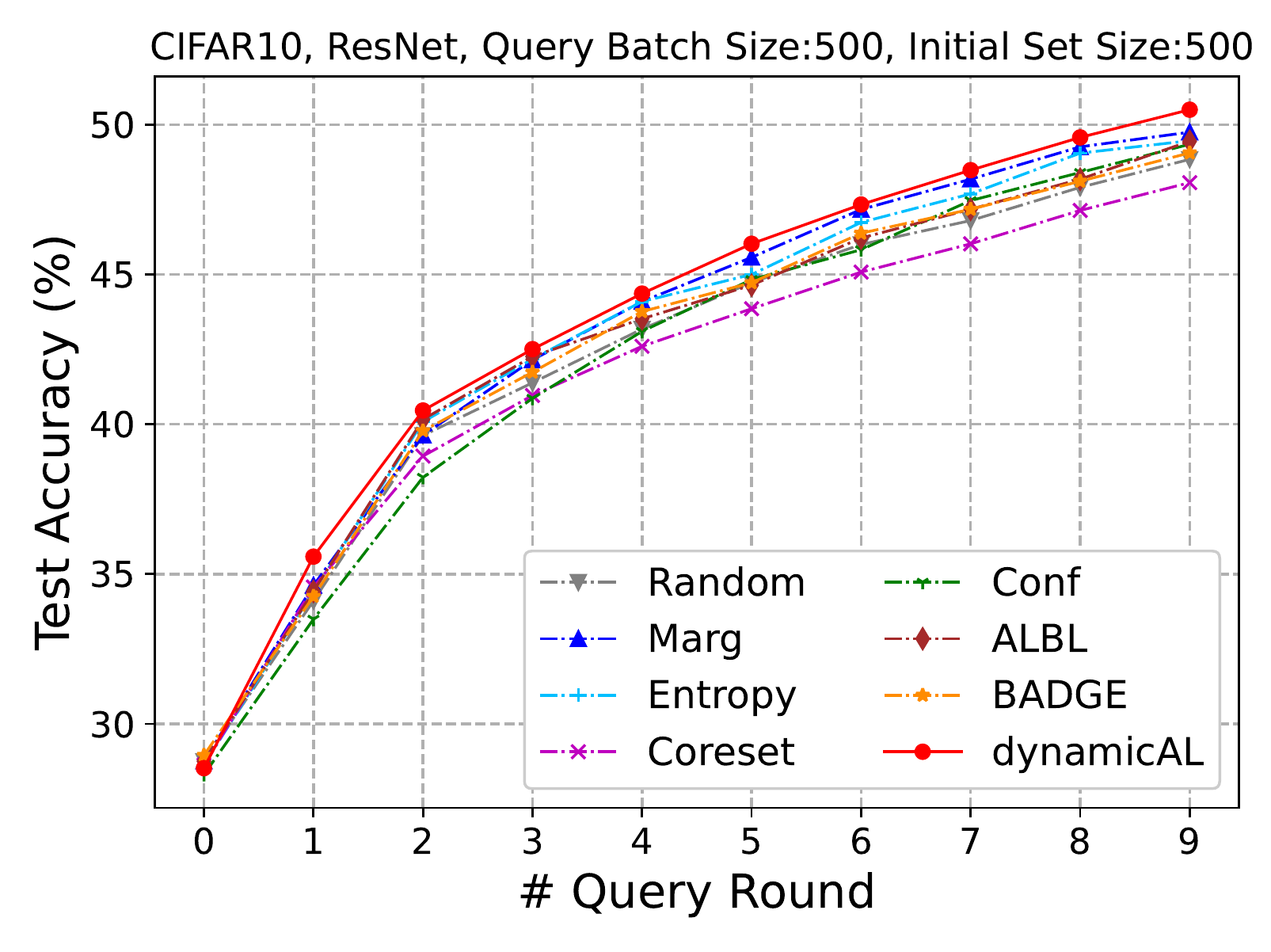}
\end{minipage}
\hspace{5pt}
\begin{minipage}{0.3\textwidth}
\includegraphics[width = 1.8in]{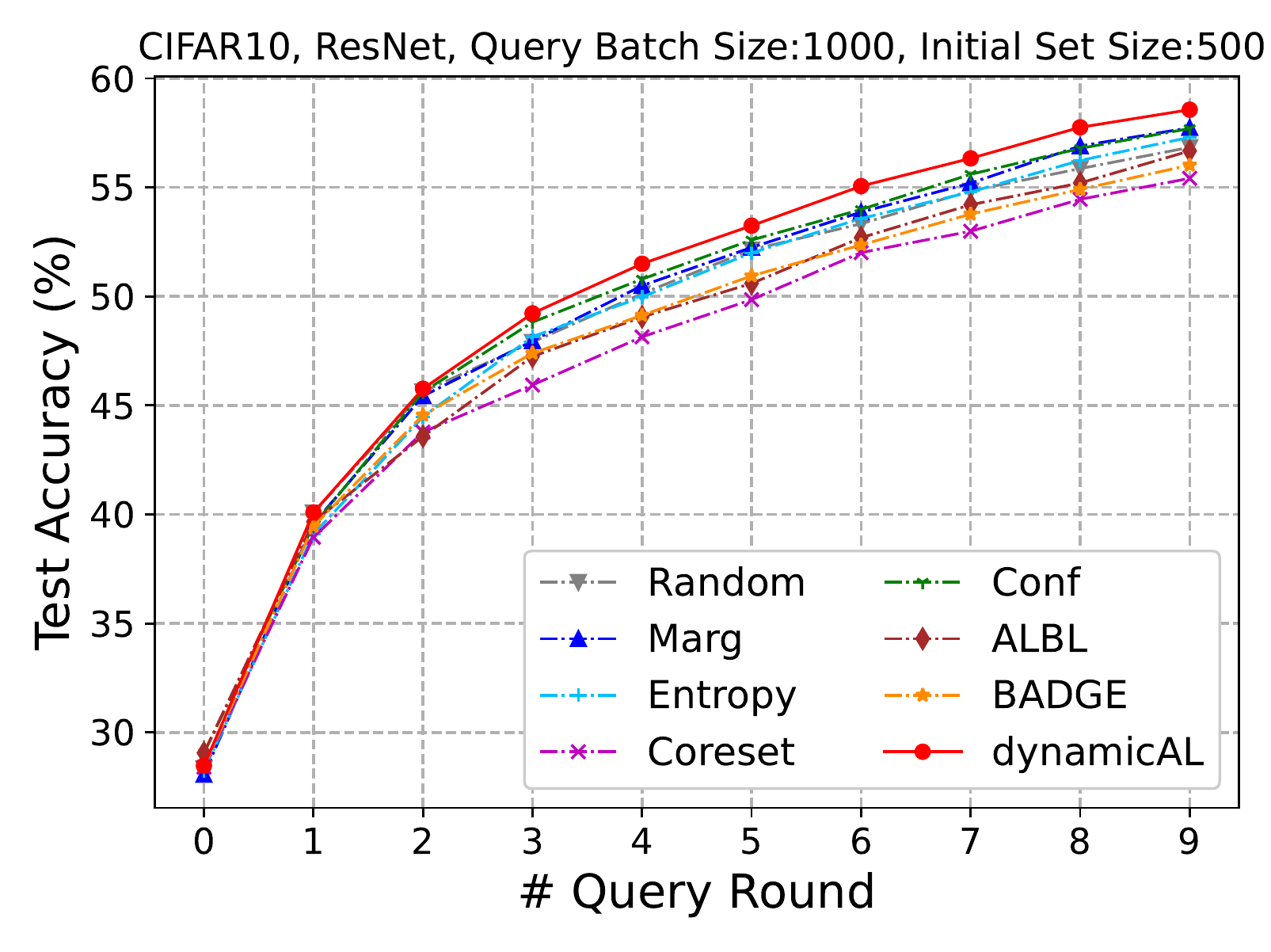}
\end{minipage}
\begin{minipage}{0.3\textwidth}
\includegraphics[width = 1.8in]{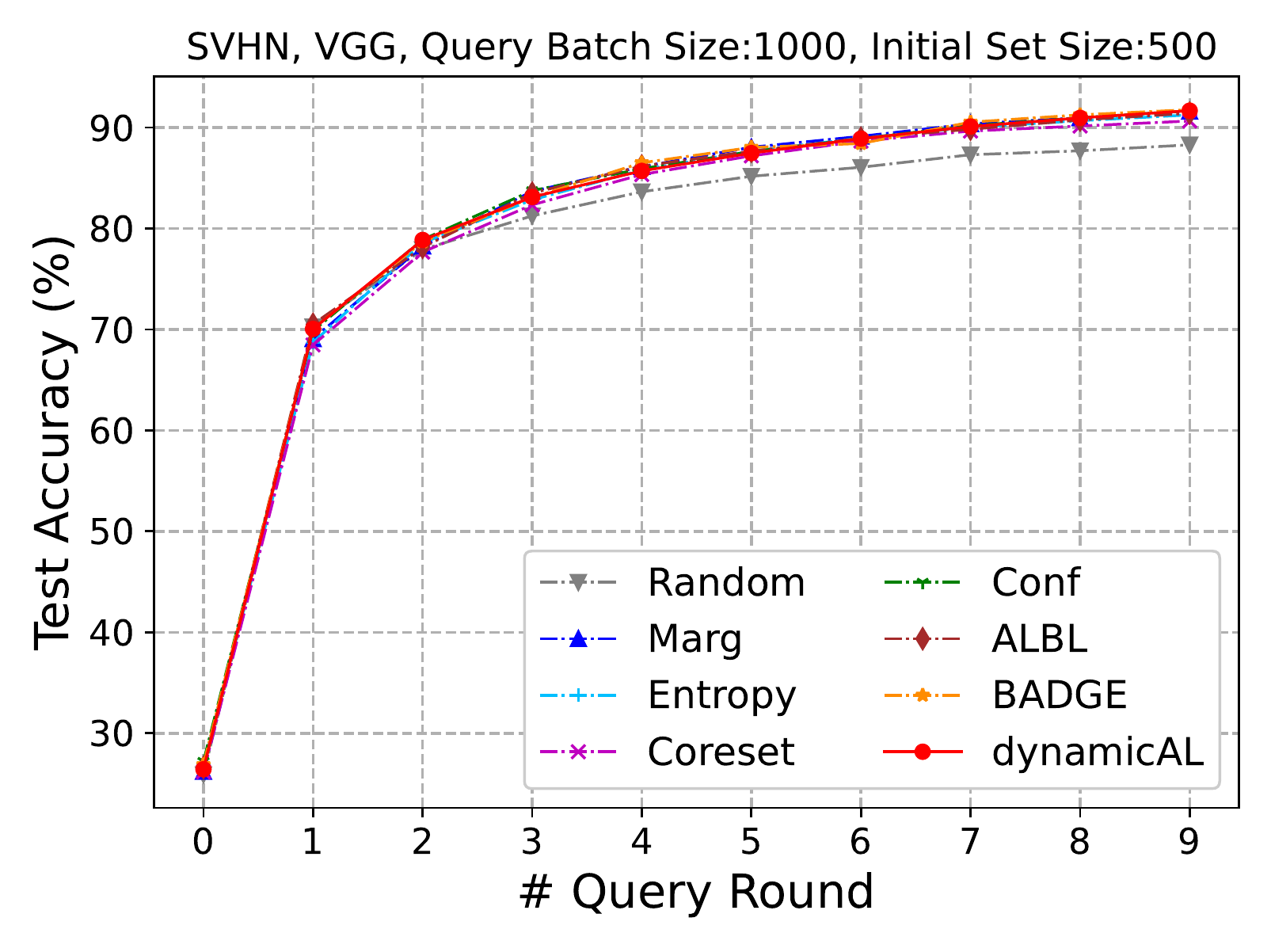}
\end{minipage}
\hspace{45pt}
\begin{minipage}{0.3\textwidth}
\includegraphics[width = 1.8in]{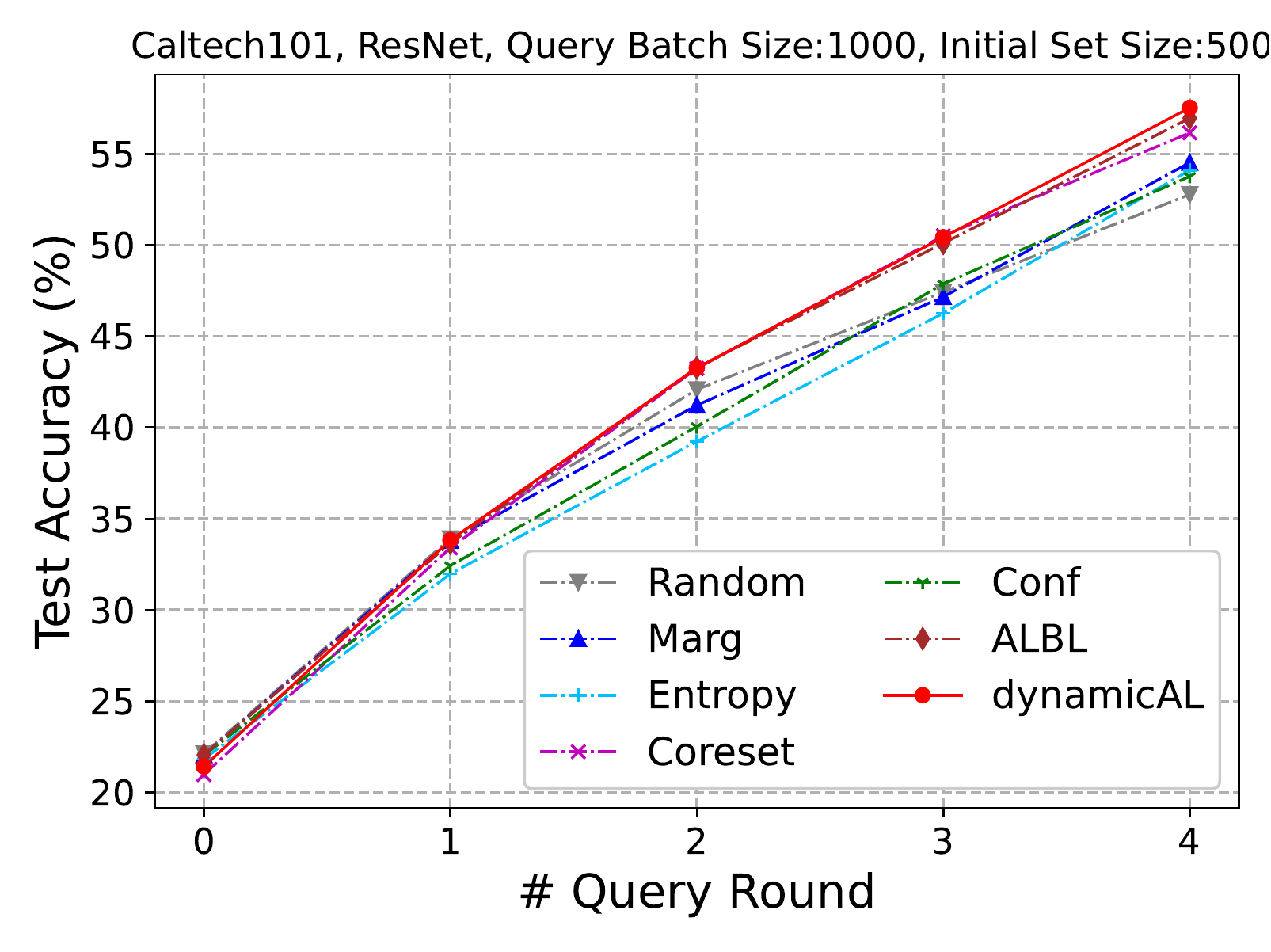}
\end{minipage}
\caption{ {The evaluation results for different active learning methods under a range of conditions.}}
\vspace{-7 pt}
\label{fig:rest_main}
\end{figure}

\subsection{Numerical Result of Main Experiments}
\label{exp_number}
For the the main experiments, we report the means and standard deviations of active learning performance under different settings in the the following tables.

\begin{table}[H]
    \centering
    \tiny
    \setlength{\tabcolsep}{4pt}{
    \caption{\color{black}{CIFAR10, CNN, Query Batch Size:250, Initial Set Size:500}}
    \begin{sc}
    \begin{tabular}{c|cccccccc}
    \toprule
    & Random & Marg &  Entropy &  Coreset & Conf &  ALBL  & BADGE  & \dgal \\ \midrule
    0 & 32.32\%$\pm$1.308\% &	32.48\%$\pm$1.286\%	& 32.32\%$\pm$1.269\% &	32.41\%$\pm$1.281\% & 32.27\%$\pm$1.266\% & 32.50\%$\pm$1.263\% & 32.66\%$\pm$1.182\%  & 32.57\%$\pm$1.423\%   \\
    1 & 33.00\%$\pm$1.175\% &	33.16\%$\pm$1.165\%	& 32.74\%$\pm$1.617\% &	32.75\%$\pm$1.306\% & 33.00\%$\pm$1.703\% & 32.98\%$\pm$1.184\% & 33.45\%$\pm$1.813\%  & 33.52\%$\pm$1.311\%   \\
    2 & 34.14\%$\pm$1.322\% &	34.41\%$\pm$1.130\%	& 34.06\%$\pm$1.546\% &	33.77\%$\pm$1.011\% & 34.21\%$\pm$1.426\% & 34.02\%$\pm$1.392\% & 34.66\%$\pm$1.483\%  & 34.70\%$\pm$1.019\%   \\
    3 & 35.05\%$\pm$1.508\% &	35.50\%$\pm$1.301\%	& 35.16\%$\pm$1.679\% &	34.44\%$\pm$0.937\% & 35.25\%$\pm$1.344\% & 34.97\%$\pm$1.227\% & 35.75\%$\pm$1.024\%  & 35.78\%$\pm$1.115\%   \\
    4 & 35.64\%$\pm$1.945\% &	36.55\%$\pm$1.249\%	& 36.14\%$\pm$1.646\% &	35.08\%$\pm$1.396\% & 36.59\%$\pm$1.508\% & 35.58\%$\pm$1.177\% & 36.33\%$\pm$0.791\%  & 36.72\%$\pm$0.716\%   \\
    5 & 36.28\%$\pm$1.124\% &	37.18\%$\pm$1.547\%	& 36.77\%$\pm$1.004\% &	35.68\%$\pm$1.390\% & 37.19\%$\pm$1.063\% & 36.15\%$\pm$1.311\% & 37.29\%$\pm$1.126\%  & 37.45\%$\pm$1.573\%   \\
    6 & 36.88\%$\pm$1.568\% &	37.73\%$\pm$1.546\%	& 37.28\%$\pm$1.983\% &	36.18\%$\pm$1.419\% & 37.65\%$\pm$2.062\% & 36.65\%$\pm$1.111\% & 37.90\%$\pm$1.988\%  & 37.95\%$\pm$1.414\%   \\
    7 & 37.29\%$\pm$1.605\% &	38.01\%$\pm$0.874\%	& 37.67\%$\pm$1.723\% &	36.57\%$\pm$1.346\% & 38.09\%$\pm$1.174\% & 37.07\%$\pm$1.731\% & 38.28\%$\pm$1.474\%  & 38.41\%$\pm$1.295\%   \\
    8 & 37.59\%$\pm$1.848\% &	38.43\%$\pm$1.675\%	& 38.01\%$\pm$1.601\% &	36.98\%$\pm$0.748\% & 38.58\%$\pm$1.556\% & 37.35\%$\pm$1.135\% & 38.51\%$\pm$1.091\%  & 38.70\%$\pm$1.291\%   \\
    9 & 37.85\%$\pm$1.789\% &	38.75\%$\pm$1.550\%	& 38.29\%$\pm$1.312\% &	37.25\%$\pm$1.527\% & 38.91\%$\pm$1.902\% & 37.57\%$\pm$1.170\% & 38.78\%$\pm$0.776\%  & 38.91\%$\pm$1.358\%   \\
    \bottomrule
\end{tabular}
\end{sc}
}
\end{table}

\begin{table}[H]
    \centering
    \tiny
    \setlength{\tabcolsep}{4pt}{
    \caption{\color{black}{CIFAR10, CNN, Query Batch Size:500, Initial Set Size:500}}
    \begin{sc}
    \begin{tabular}{c|cccccccc}
    \toprule
    &  Random &  Marg &  Entropy &  Coreset & Conf &  ALBL  &  BADGE  &  \dgal \\ \midrule
    0 & 32.26\%$\pm$1.164\% &	32.31\%$\pm$1.441\%	& 32.29\%$\pm$1.397\% &	32.54\%$\pm$1.331\% & 32.32\%$\pm$1.288\% & 32.41\%$\pm$1.432\% & 32.49\%$\pm$1.320\%  & 32.37\%$\pm$1.049\%   \\
    1 & 34.87\%$\pm$1.286\% &	34.89\%$\pm$1.575\%	& 34.58\%$\pm$1.664\% &	33.84\%$\pm$1.368\% & 34.75\%$\pm$1.503\% & 34.08\%$\pm$1.368\% & 34.44\%$\pm$1.230\%  & 34.88\%$\pm$1.557\%   \\
    2 & 36.45\%$\pm$0.842\% &	36.69\%$\pm$1.456\%	& 36.50\%$\pm$1.463\% &	35.96\%$\pm$1.667\% & 36.73\%$\pm$1.744\% & 35.62\%$\pm$1.536\% & 36.41\%$\pm$1.175\%  & 36.78\%$\pm$1.253\%   \\
    3 & 37.16\%$\pm$0.767\% &	37.99\%$\pm$1.356\%	& 37.30\%$\pm$1.221\% &	36.20\%$\pm$1.086\% & 38.12\%$\pm$1.663\% & 36.55\%$\pm$1.327\% & 37.56\%$\pm$1.284\%  & 38.30\%$\pm$1.152\%   \\
    4 & 37.89\%$\pm$0.880\% &	39.15\%$\pm$1.056\%	& 38.23\%$\pm$0.878\% &	36.73\%$\pm$1.011\% & 39.10\%$\pm$1.336\% & 37.20\%$\pm$1.381\% & 38.49\%$\pm$1.238\%  & 39.37\%$\pm$0.708\%   \\
    5 & 38.59\%$\pm$0.861\% &	39.98\%$\pm$1.562\%	& 39.01\%$\pm$1.278\% &	37.33\%$\pm$1.373\% & 39.81\%$\pm$1.402\% & 37.80\%$\pm$1.560\% & 39.61\%$\pm$1.219\%  & 40.09\%$\pm$0.940\%   \\
    6 & 39.15\%$\pm$1.108\% &	40.70\%$\pm$1.391\%	& 39.68\%$\pm$1.315\% &	37.97\%$\pm$1.393\% & 40.47\%$\pm$1.126\% & 38.47\%$\pm$1.270\% & 40.55\%$\pm$1.066\%  & 40.75\%$\pm$1.671\%   \\
    7 & 39.51\%$\pm$1.219\% &	40.99\%$\pm$1.217\%	& 40.09\%$\pm$1.408\% &	38.53\%$\pm$1.600\% & 41.05\%$\pm$1.448\% & 39.11\%$\pm$1.385\% & 40.97\%$\pm$0.814\%  & 41.21\%$\pm$1.433\%   \\
    8 & 39.90\%$\pm$0.807\% &	41.39\%$\pm$1.614\%	& 40.39\%$\pm$1.357\% &	39.06\%$\pm$1.156\% & 41.30\%$\pm$1.865\% & 39.55\%$\pm$1.595\% & 41.27\%$\pm$1.409\%  & 41.59\%$\pm$1.013\%   \\
    9 & 40.17\%$\pm$1.170\% &	41.64\%$\pm$1.287\%	& 40.71\%$\pm$0.739\% &	39.43\%$\pm$0.892\% & 41.55\%$\pm$1.341\% & 39.95\%$\pm$1.299\% & 41.41\%$\pm$0.949\%  & 41.78\%$\pm$0.645\%   \\
    \bottomrule
\end{tabular}
\end{sc}}
\end{table}

\begin{table}[H]
    \centering
    \tiny
    \setlength{\tabcolsep}{4pt}{
    \caption{\color{black}CIFAR10, ResNet, Query Batch Size:500, Initial Set Size:500}
    \begin{sc}
    \begin{tabular}{c|cccccccc}
    \toprule
    &  Random &  Marg &  Entropy &  Coreset &  Conf &  ALBL  &  BADGE  &  \dgal \\ \midrule
    0 & 28.75\%$\pm$1.780\% &	28.75\%$\pm$0.369\%	& 28.63\%$\pm$1.394\% &	28.63\%$\pm$1.120\% & 28.31\%$\pm$1.011\% & 28.75\%$\pm$0.957\% & 28.95\%$\pm$1.040\%  & 28.52\%$\pm$0.686\%   \\
    1 & 34.11\%$\pm$3.088\% &	34.62\%$\pm$2.022\%	& 34.42\%$\pm$0.849\% &	34.57\%$\pm$0.992\% & 33.49\%$\pm$1.269\% & 34.42\%$\pm$2.077\% & 34.26\%$\pm$1.740\%  & 35.58\%$\pm$2.858\%   \\
    2 & 39.63\%$\pm$2.157\% &	39.63\%$\pm$0.313\%	& 40.08\%$\pm$1.022\% &	38.94\%$\pm$1.408\% & 38.23\%$\pm$1.454\% & 40.16\%$\pm$2.574\% & 39.78\%$\pm$1.384\%  & 40.46\%$\pm$0.959\%   \\
    3 & 41.38\%$\pm$2.357\% &	42.15\%$\pm$0.810\%	& 42.18\%$\pm$1.271\% &	40.96\%$\pm$0.961\% & 40.87\%$\pm$0.860\% & 42.26\%$\pm$2.347\% & 41.74\%$\pm$1.230\%  & 42.51\%$\pm$0.799\%   \\
    4 & 43.18\%$\pm$1.809\% &	44.09\%$\pm$1.165\%	& 44.09\%$\pm$1.150\% &	42.60\%$\pm$1.094\% & 43.10\%$\pm$1.325\% & 43.52\%$\pm$3.064\% & 43.76\%$\pm$1.364\%  & 44.36\%$\pm$0.980\%   \\
    5 & 44.73\%$\pm$2.253\% &	45.57\%$\pm$1.115\%	& 45.00\%$\pm$0.731\% &	43.86\%$\pm$1.369\% & 44.83\%$\pm$1.388\% & 44.64\%$\pm$3.097\% & 44.73\%$\pm$1.675\%  & 46.02\%$\pm$0.754\%   \\
    6 & 46.00\%$\pm$2.193\% &	47.17\%$\pm$0.929\%	& 46.74\%$\pm$1.118\% &	45.08\%$\pm$1.549\% & 45.83\%$\pm$1.426\% & 46.22\%$\pm$2.601\% & 46.38\%$\pm$1.607\%  & 47.34\%$\pm$1.027\%   \\
    7 & 46.80\%$\pm$2.134\% &	48.18\%$\pm$1.230\%	& 47.69\%$\pm$1.253\% &	46.02\%$\pm$1.589\% & 47.47\%$\pm$1.424\% & 47.18\%$\pm$2.384\% & 47.17\%$\pm$1.404\%  & 48.48\%$\pm$1.452\%   \\
    8 & 47.91\%$\pm$1.722\% &	49.26\%$\pm$0.652\%	& 49.05\%$\pm$1.113\% &	47.14\%$\pm$1.880\% & 48.40\%$\pm$1.178\% & 48.18\%$\pm$2.503\% & 48.11\%$\pm$1.049\%  & 49.58\%$\pm$1.673\%   \\
    9 & 48.84\%$\pm$1.584\% &	49.75\%$\pm$1.341\%	& 49.46\%$\pm$1.282\% &	48.07\%$\pm$1.480\% & 49.35\%$\pm$1.269\% & 49.45\%$\pm$2.529\% & 49.06\%$\pm$0.850\%  & 50.50\%$\pm$1.301\%   \\
    \bottomrule
\end{tabular}
\end{sc}}
\end{table}

\begin{table}[H]
    \centering
    \tiny
    \setlength{\tabcolsep}{4pt}{
    \caption{\color{black}CIFAR10, ResNet, Query Batch Size:1000, Initial Set Size:500}
    \begin{sc}
    \begin{tabular}{c|cccccccc}
    \toprule
    &  Random &  Marg &  Entropy &  Coreset &  Conf &  ALBL  &  BADGE  &  \dgal \\ \midrule
    0 & 28.34\%$\pm$1.465\% &	28.07\%$\pm$2.604\%	& 28.24\%$\pm$1.756\% &	28.41\%$\pm$0.722\% & 29.05\%$\pm$1.137\% & 29.06\%$\pm$0.847\% & 28.43\%$\pm$1.176\%  & 28.48\%$\pm$2.062\%   \\
    1 & 40.08\%$\pm$0.329\% &	39.57\%$\pm$1.551\%	& 39.09\%$\pm$2.180\% &	38.95\%$\pm$1.047\% & 39.50\%$\pm$2.340\% & 39.67\%$\pm$1.489\% & 39.46\%$\pm$3.020\%  & 40.09\%$\pm$1.795\%   \\
    2 & 45.63\%$\pm$1.253\% &	45.43\%$\pm$0.444\%	& 44.48\%$\pm$1.823\% &	43.78\%$\pm$0.986\% & 45.62\%$\pm$1.882\% & 43.58\%$\pm$1.329\% & 44.55\%$\pm$3.654\%  & 45.77\%$\pm$2.290\%   \\
    3 & 47.90\%$\pm$1.257\% &	47.96\%$\pm$0.735\%	& 48.15\%$\pm$2.509\% &	45.93\%$\pm$0.682\% & 48.82\%$\pm$1.797\% & 47.24\%$\pm$1.926\% & 47.39\%$\pm$4.189\%  & 49.22\%$\pm$1.704\%   \\
    4 & 50.13\%$\pm$1.207\% &	50.49\%$\pm$0.807\%	& 49.97\%$\pm$2.819\% &	48.14\%$\pm$0.566\% & 50.79\%$\pm$1.870\% & 49.05\%$\pm$1.831\% & 49.13\%$\pm$4.053\%  & 51.50\%$\pm$1.925\%   \\
    5 & 52.14\%$\pm$1.517\% &	52.24\%$\pm$0.781\%	& 52.00\%$\pm$2.762\% &	49.85\%$\pm$1.075\% & 52.59\%$\pm$2.202\% & 50.59\%$\pm$1.636\% & 50.94\%$\pm$3.628\%  & 53.24\%$\pm$1.927\%   \\
    6 & 53.33\%$\pm$1.300\% &	53.87\%$\pm$0.635\%	& 53.57\%$\pm$3.123\% &	52.01\%$\pm$0.772\% & 53.99\%$\pm$2.390\% & 52.69\%$\pm$1.599\% & 52.36\%$\pm$3.924\%  & 55.06\%$\pm$1.697\%   \\
    7 & 54.84\%$\pm$1.238\% &	55.19\%$\pm$1.136\%	& 54.79\%$\pm$3.144\% &	52.99\%$\pm$1.147\% & 55.60\%$\pm$2.002\% & 54.20\%$\pm$1.685\% & 53.77\%$\pm$3.985\%  & 56.33\%$\pm$1.613\%   \\
    8 & 55.86\%$\pm$1.161\% &	56.90\%$\pm$0.732\%	& 56.23\%$\pm$3.182\% &	54.45\%$\pm$0.821\% & 56.79\%$\pm$2.033\% & 55.20\%$\pm$1.868\% & 54.91\%$\pm$4.104\%  & 57.76\%$\pm$1.796\%   \\
    9 & 56.84\%$\pm$0.979\% &	57.73\%$\pm$0.500\%	& 57.29\%$\pm$3.225\% &	55.42\%$\pm$0.954\% & 57.70\%$\pm$2.042\% & 56.67\%$\pm$1.783\% & 56.02\%$\pm$3.935\%  & 58.56\%$\pm$1.574\%   \\
    \bottomrule
\end{tabular}
\end{sc}}
\end{table}

\begin{table}[H]
    \centering
    \tiny
    \setlength{\tabcolsep}{4pt}{
    \caption{\color{black}SVHN, VGG, Query Batch Size:500, Initial Set Size:500}
    \begin{sc}
    \begin{tabular}{c|cccccccc}
    \toprule
    &  Random &  Marg &  Entropy &  Coreset &  Conf & ALBL  &  BADGE  &  \dgal \\ \midrule
    0 & 25.94\%$\pm$7.158\% &	26.15\%$\pm$6.290\%	& 26.41\%$\pm$8.994\% &	25.83\%$\pm$5.845\% & 26.52\%$\pm$7.489\% & 25.31\%$\pm$5.030\% & 26.38\%$\pm$9.100\%  & 26.30\%$\pm$5.505\%   \\
    1 & 61.23\%$\pm$4.812\% &	63.93\%$\pm$3.127\%	& 59.02\%$\pm$4.724\% &	57.02\%$\pm$3.672\% & 61.99\%$\pm$2.613\% & 62.14\%$\pm$5.531\% & 58.70\%$\pm$5.615\%  & 63.01\%$\pm$12.293\%   \\
    2 & 71.35\%$\pm$2.364\% &	74.08\%$\pm$0.933\%	& 71.25\%$\pm$1.459\% &	67.95\%$\pm$2.870\% & 73.31\%$\pm$2.828\% & 71.75\%$\pm$2.555\% & 74.74\%$\pm$2.978\%  & 74.10\%$\pm$4.557\%   \\
    3 & 76.34\%$\pm$1.626\% &	79.17\%$\pm$1.064\%	& 76.74\%$\pm$1.521\% &	73.76\%$\pm$2.844\% & 78.02\%$\pm$1.939\% & 77.65\%$\pm$1.518\% & 77.75\%$\pm$2.100\%  & 79.20\%$\pm$2.651\%   \\
    4 & 78.86\%$\pm$1.378\% &	82.18\%$\pm$0.504\%	& 79.67\%$\pm$0.809\% &	78.14\%$\pm$2.486\% & 81.32\%$\pm$1.901\% & 81.09\%$\pm$1.005\% & 80.16\%$\pm$1.353\%  & 82.33\%$\pm$2.134\%   \\
    5 & 80.56\%$\pm$1.149\% &	83.85\%$\pm$0.750\%	& 81.87\%$\pm$0.638\% &	80.34\%$\pm$2.339\% & 83.31\%$\pm$1.529\% & 83.37\%$\pm$1.225\% & 82.94\%$\pm$0.830\%  & 84.19\%$\pm$1.940\%   \\
    6 & 81.98\%$\pm$1.334\% &	85.61\%$\pm$0.624\%	& 83.56\%$\pm$0.541\% &	82.32\%$\pm$1.592\% & 84.94\%$\pm$0.858\% & 85.19\%$\pm$0.993\% & 83.69\%$\pm$0.975\%  & 85.80\%$\pm$1.498\%   \\
    7 & 83.00\%$\pm$1.048\% &	86.62\%$\pm$0.607\%	& 84.94\%$\pm$0.079\% &	83.98\%$\pm$1.394\% & 85.97\%$\pm$1.179\% & 86.31\%$\pm$0.977\% & 85.15\%$\pm$0.760\%  & 86.75\%$\pm$1.426\%   \\
    8 & 83.59\%$\pm$0.945\% &	87.57\%$\pm$0.625\%	& 85.78\%$\pm$0.068\% &	85.26\%$\pm$1.431\% & 87.13\%$\pm$0.679\% & 87.55\%$\pm$0.831\% & 86.61\%$\pm$0.478\%  & 87.91\%$\pm$1.264\%   \\
    9 & 84.42\%$\pm$0.744\% &	88.23\%$\pm$0.600\%	& 87.11\%$\pm$0.437\% &	86.18\%$\pm$0.886\% & 87.87\%$\pm$0.598\% & 87.89\%$\pm$0.780\% & 87.29\%$\pm$0.441\%  & 88.52\%$\pm$1.240\%   \\
    \bottomrule
\end{tabular}
\end{sc}}
\end{table}

\begin{table}[H]
    \centering
    \tiny
    \setlength{\tabcolsep}{4pt}{
    \caption{\color{black}SVHN, VGG, Query Batch Size:1000, Initial Set Size:500}
    \begin{sc}
    \begin{tabular}{c|cccccccc}
    \toprule
    &  Random &  Marg &  Entropy &  Coreset &  Conf &  ALBL  &  BADGE  &  \dgal \\ \midrule
    0 & 25.90\%$\pm$3.479\% &	26.20\%$\pm$5.409\%	& 26.85\%$\pm$4.403\% &	26.18\%$\pm$6.853\% & 27.21\%$\pm$8.721\% & 26.60\%$\pm$4.688\% & 26.88\%$\pm$6.248\%  & 26.43\%$\pm$8.047\%   \\
    1 & 70.26\%$\pm$3.154\% &	69.06\%$\pm$3.646\%	& 68.72\%$\pm$2.156\% &	68.46\%$\pm$1.111\% & 69.85\%$\pm$3.485\% & 70.51\%$\pm$3.487\% & 70.09\%$\pm$2.690\%  & 70.04\%$\pm$1.650\%   \\
    2 & 77.91\%$\pm$1.061\% &	78.24\%$\pm$2.237\%	& 78.56\%$\pm$0.492\% &	77.66\%$\pm$1.784\% & 78.89\%$\pm$2.809\% & 78.14\%$\pm$1.494\% & 78.67\%$\pm$1.799\%  & 78.86\%$\pm$1.710\%   \\
    3 & 81.25\%$\pm$0.812\% &	83.68\%$\pm$1.657\%	& 82.83\%$\pm$0.527\% &	82.34\%$\pm$1.461\% & 83.75\%$\pm$2.165\% & 83.50\%$\pm$1.669\% & 83.07\%$\pm$1.334\%  & 83.11\%$\pm$1.269\%   \\
    4 & 83.63\%$\pm$0.746\% &	86.12\%$\pm$1.251\%	& 85.80\%$\pm$0.744\% &	85.34\%$\pm$1.126\% & 85.91\%$\pm$1.128\% & 86.18\%$\pm$0.979\% & 86.50\%$\pm$1.087\%  & 85.70\%$\pm$1.179\%   \\
    5 & 85.17\%$\pm$0.870\% &	88.04\%$\pm$1.022\%	& 87.66\%$\pm$0.683\% &	87.19\%$\pm$0.928\% & 87.61\%$\pm$1.044\% & 87.65\%$\pm$1.031\% & 88.03\%$\pm$0.742\%  & 87.46\%$\pm$1.054\%   \\
    6 & 86.06\%$\pm$0.822\% &	89.13\%$\pm$0.712\%	& 88.96\%$\pm$0.395\% &	88.65\%$\pm$0.505\% & 88.90\%$\pm$0.845\% & 88.93\%$\pm$0.809\% & 88.41\%$\pm$0.783\%  & 88.89\%$\pm$1.274\%   \\
    7 & 87.30\%$\pm$0.948\% &	90.36\%$\pm$0.532\%	& 90.00\%$\pm$0.257\% &	89.65\%$\pm$0.486\% & 90.18\%$\pm$0.706\% & 89.83\%$\pm$0.747\% & 90.53\%$\pm$0.495\%  & 90.09\%$\pm$1.149\%   \\
    8 & 87.69\%$\pm$0.890\% &	90.95\%$\pm$0.375\%	& 90.67\%$\pm$0.385\% &	90.15\%$\pm$0.410\% & 90.96\%$\pm$0.677\% & 90.75\%$\pm$0.567\% & 91.25\%$\pm$0.432\%  & 90.95\%$\pm$0.782\%   \\
    9 & 88.28\%$\pm$0.723\% &	91.59\%$\pm$0.417\%	& 91.25\%$\pm$0.353\% &	90.64\%$\pm$0.311\% & 91.66\%$\pm$0.755\% & 91.41\%$\pm$0.665\% & 91.76\%$\pm$0.367\%  & 91.67\%$\pm$0.840\%   \\
    \bottomrule
\end{tabular}
\end{sc}}
\end{table}


\begin{table}[H]
    \centering
    \tiny
    \setlength{\tabcolsep}{4pt}{
    \caption{\color{black}{Caltech101, ResNet, Query Batch Size:500, Initial Set Size:500}}
    \begin{sc}
    \begin{tabular}{c|ccccccc}
    \toprule
    &  Random &  Marg &  Entropy &  Coreset &  Conf &  ALBL  &  \dgal \\ \midrule
    0 & 21.59\%$\pm$1.431\% &	21.98\%$\pm$1.688\%	& 21.49\%$\pm$1.681\% &	21.39\%$\pm$1.738\% & 21.98\%$\pm$1.459\% & 21.48\%$\pm$1.828\% & 21.38\%$\pm$1.323\%   \\
    1 & 25.84\%$\pm$1.112\% &	28.42\%$\pm$1.677\%	& 27.43\%$\pm$0.760\% &	28.61\%$\pm$1.224\% & 27.25\%$\pm$1.155\% & 28.02\%$\pm$1.515\% & 29.34\%$\pm$2.111\%   \\
    2 & 34.94\%$\pm$0.635\% &	34.76\%$\pm$1.745\%	& 32.94\%$\pm$1.224\% &	35.37\%$\pm$1.561\% & 32.85\%$\pm$0.849\% & 35.86\%$\pm$1.012\% & 36.14\%$\pm$1.234\%   \\
    3 & 37.34\%$\pm$1.088\% &	39.70\%$\pm$1.328\%	& 36.36\%$\pm$0.636\% &	40.81\%$\pm$1.005\% & 37.52\%$\pm$1.250\% & 40.20\%$\pm$1.091\% & 41.19\%$\pm$0.789\%   \\
    4 & 43.87\%$\pm$0.867\% &	43.26\%$\pm$0.612\%	& 40.12\%$\pm$0.805\% &	45.38\%$\pm$0.508\% & 41.82\%$\pm$1.104\% & 45.27\%$\pm$0.725\% & 46.11\%$\pm$1.138\%   \\
    5 & 45.45\%$\pm$1.672\% &	46.25\%$\pm$1.562\%	& 43.24\%$\pm$1.617\% &	47.81\%$\pm$1.683\% & 44.60\%$\pm$1.295\% & 48.35\%$\pm$1.729\% & 49.42\%$\pm$1.298\%   \\
    6 & 47.60\%$\pm$1.383\% &	49.20\%$\pm$1.310\%	& 45.71\%$\pm$1.047\% &	50.60\%$\pm$1.596\% & 46.74\%$\pm$0.760\% & 51.20\%$\pm$1.466\% & 52.31\%$\pm$1.739\%   \\
    7 & 49.97\%$\pm$0.530\% &	51.40\%$\pm$1.571\%	& 48.19\%$\pm$0.928\% &	52.80\%$\pm$1.887\% & 49.19\%$\pm$0.885\% & 53.90\%$\pm$1.166\% & 55.03\%$\pm$1.098\%   \\
    8 & 52.06\%$\pm$1.476\% &	53.56\%$\pm$1.044\%	& 50.81\%$\pm$0.943\% &	55.31\%$\pm$1.105\% & 51.99\%$\pm$1.383\% & 56.22\%$\pm$0.838\% & 56.92\%$\pm$1.153\%   \\
    9 & 54.04\%$\pm$0.898\% &	55.92\%$\pm$0.496\%	& 53.05\%$\pm$0.554\% &	56.93\%$\pm$0.691\% & 54.96\%$\pm$0.981\% & 57.99\%$\pm$0.805\% & 58.81\%$\pm$1.040\%   \\
    \bottomrule
\end{tabular}
\end{sc}}
\end{table}

\begin{table}[H]
    \centering
    \tiny
    \setlength{\tabcolsep}{4pt}{
    \caption{\color{black}Caltech101, ResNet, Query Batch Size:1000, Initial Set Size:500}
    \begin{sc}
    \begin{tabular}{c|ccccccc}
    \toprule
    &  Random &  Marg &  Entropy &  Coreset &  Conf &  ALBL  &  \dgal \\ \midrule
    0 & 22.13\%$\pm$1.050\% &	22.05\%$\pm$1.011\%	& 21.83\%$\pm$0.725\% &	20.98\%$\pm$0.631\% & 22.03\%$\pm$1.364\% & 22.05\%$\pm$0.633\% & 21.42\%$\pm$1.735\%  \\
    1 & 33.91\%$\pm$1.330\% &	33.80\%$\pm$1.002\%	& 31.98\%$\pm$1.000\% &	33.40\%$\pm$0.962\% & 32.43\%$\pm$0.895\% & 33.66\%$\pm$2.174\% & 33.83\%$\pm$1.438\%  \\
    2 & 42.08\%$\pm$0.560\% &	41.22\%$\pm$0.730\%	& 39.23\%$\pm$0.981\% &	43.24\%$\pm$0.960\% & 40.05\%$\pm$0.988\% & 43.28\%$\pm$2.360\% & 43.27\%$\pm$2.280\%  \\
    3 & 47.43\%$\pm$0.700\% &	47.16\%$\pm$0.659\%	& 46.26\%$\pm$0.968\% &	50.51\%$\pm$0.706\% & 47.87\%$\pm$0.698\% & 50.10\%$\pm$2.082\% & 50.43\%$\pm$1.634\%  \\
    4 & 52.77\%$\pm$0.980\% &	54.52\%$\pm$1.288\%	& 54.11\%$\pm$1.347\% &	56.15\%$\pm$1.284\% & 53.76\%$\pm$1.196\% & 56.96\%$\pm$1.733\% & 57.52\%$\pm$1.189\%  \\
    \bottomrule
\end{tabular}
\end{sc}}
\end{table}

\color{black}
\subsection{Maximum Mean Discrepancy for Multiple Rounds}
\label{mmd_multi_round}
As shown in Figure~\ref{fig:MMD_B}, the MMD term is much smaller than the $\mathcal{B}$ at the first query round. To better understand the relation between MMD and $\mathcal{B}$ for multiple query setting, we measure the $\text{MMD}/\mathcal{B}$ for $R \geq 2$. As shown in Figure~\ref{fig:mmd_b_large_R}, $\mathcal{B}$ is much larger than $\text{MMD}$ even  multiple query rounds. 
Besides, we notice that, for the first round, the larger query batch always leads to larger $\text{MMD}/\mathcal{B}$, because the sampling bias introduced by the query policy will be amplified by using large batch size. 

\begin{figure}[h]
    \centering
    \includegraphics[width=0.5\textwidth]{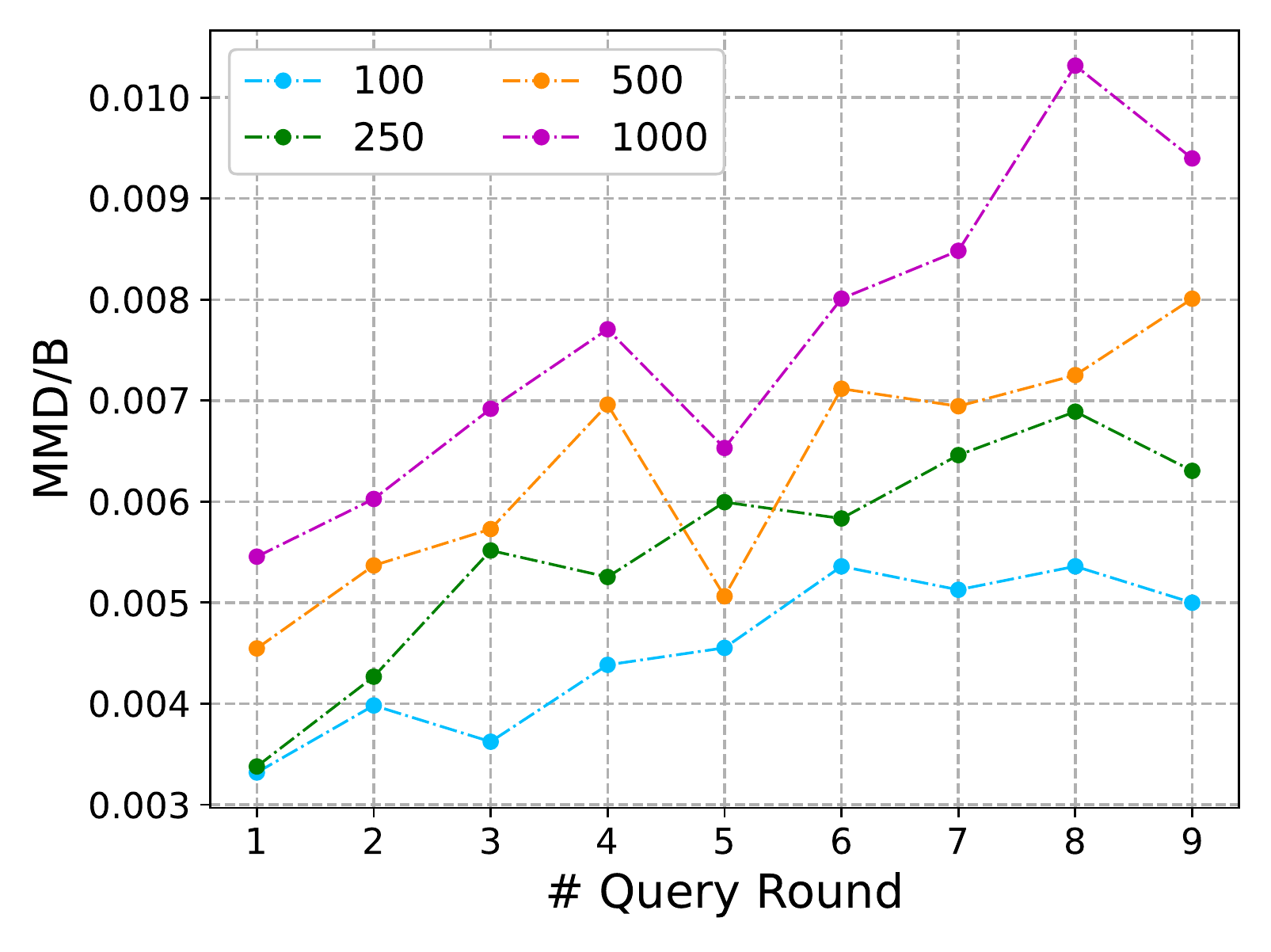}
    \vspace{-12pt}
    \caption{\color{black} $\text{MMD}/\mathcal{B}$ for larger query round.}
    \label{fig:mmd_b_large_R}
\end{figure}

Furthermore, we measure the $\text{MMD}/\mathcal{B}$ with a constant total budget size but different query rounds. The result is shown in Table~\ref{tab:mmd_b_constant_budget}. As our expectation, spending the total query budget in one query round will induce the largest $\text{MMD}/\mathcal{B}$. And, with more query rounds, the $\text{MMD}/\mathcal{B}$ will be lower.

\begin{table}[H]
\centering
\small
\setlength{\tabcolsep}{4pt}{
\begin{sc}
\centering
\caption{ \color{black}  $\text{MMD}/\mathcal{B}$ under constant budget size.}
\begin{tabular}{ccccc}
\toprule
{Setting } &\textbf{$R=10, b=100$}&\textbf{$R=4, b=250$} &\textbf{$R=2, b=500$ } &\textbf{$R=1, b=1000$ } \\
\midrule
$\text{MMD}/\mathcal{B}$   & {0.004999} & {0.005253} & {0.005367} & {0.005455} \\
\bottomrule
\label{tab:mmd_b_constant_budget}
\end{tabular}
\end{sc}}
\vspace{-10pt}
\end{table}

\color{black}

\color{black}
\subsection{Performance under the Re-initialization Setting}
\label{retrain_exp}
{\color{black} To study the effectiveness of \dgal under the re-initialization setting, we compare \dgal\ with the strong baseline involving the re-initialization trick in its algorithm design, e.g., Coreset~\citep{sener2017active}.
Following~\citep{ash2019deep}, we query samples when training accuracy is greater than $99\%$ and the results are summarized in Table~\ref{tab:retrain_cf10} and~\ref{tab:retrain_svhn}. 
The results show that \dgal\ can still be better than or competitive with the commonly used active learning methods. 
We notice that the improvement in the non-retraining setting is more significant. 
This is as our expectation. The dynamic analysis (Equation~\eqref{equ:gradient}), that \dgal\ is based on, considers the change of dynamics according to the model's current parameters. The re-initialization trick will not only causes the computational overhead of retraining, but also makes \dgal\ deviate from the analysis (Section~\ref{theo}).}

\color{black}
\subsection{Performance with large query rounds}
We provide the experiments with $b=500, r=15$ on Caltech101 data set with ResNet18 as the backbone. We ignore the BACKGROUND Google label and then we have 8677 images in total. At the last round, we run out of all images in the pool. As shown in the Figure~\ref{fig:large_R_caltech101}, our method consistently outperforms those baselines. Note, due to the non-retraining setting, the model will have different performance even if all the samples are used for the training.

\begin{figure}[h]
    \centering
    \includegraphics[width=0.5\textwidth]{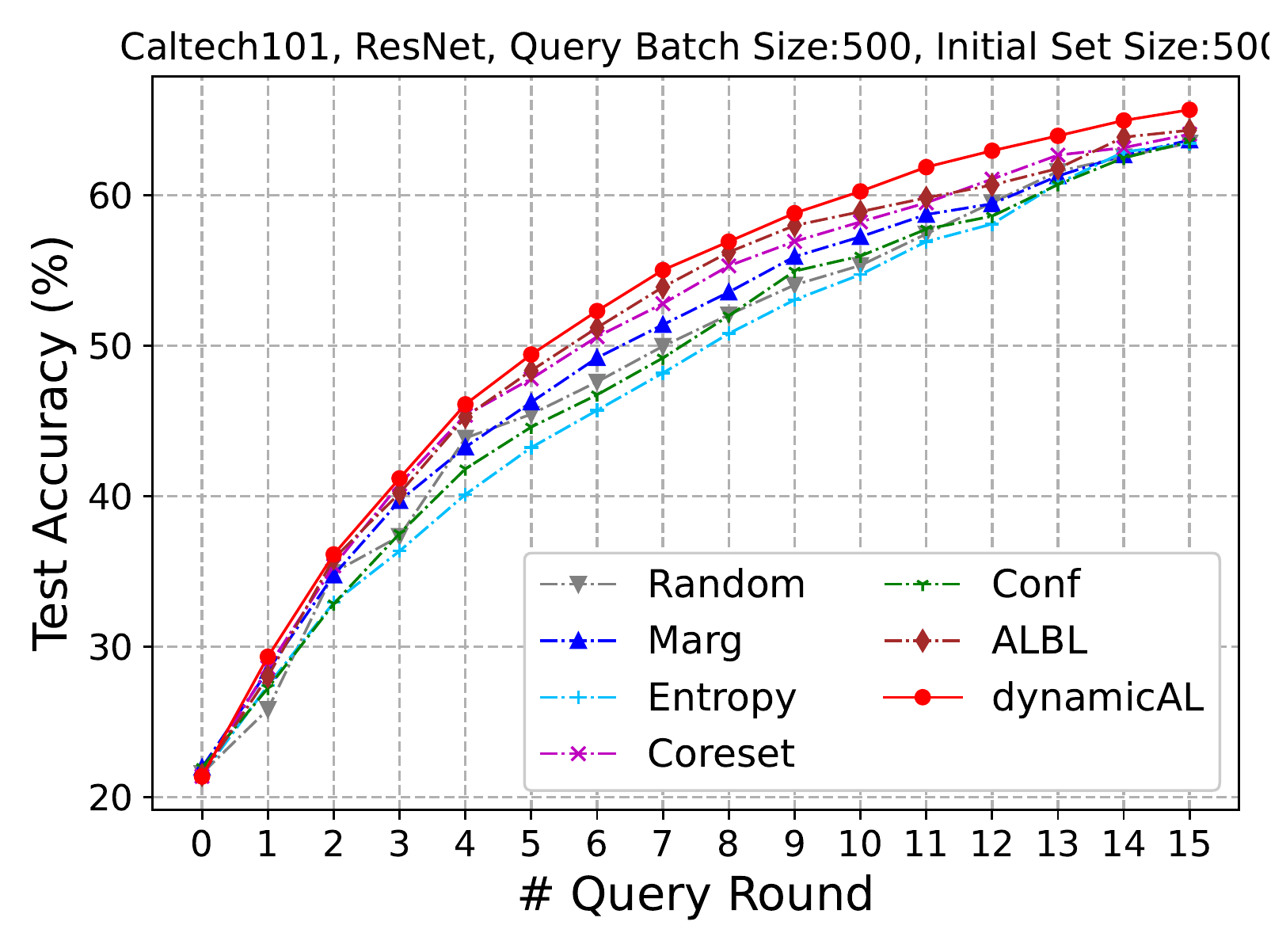}
    \vspace{-12pt}
    \caption{\color{black} The evaluation result with larger query round on Caltech101.}
    \label{fig:large_R_caltech101}
\end{figure}
\color{black}

\begin{table}[H]
\centering
\small
\setlength{\tabcolsep}{4pt}{
\begin{sc}
\centering
\caption{\color{black} CIFAR10, ResNet, Query Batch Size 500, Initial Set Size 500.}
\begin{tabular}{cccc}
\toprule
{\#Round} &{Random}&{Coreset} &{\dgal} \\
\midrule
0 & 30.80$\pm$1.81  & 30.77$\pm$0.92 & 30.94$\pm$2.17 \\
1 & 35.80$\pm$1.52  & 36.62$\pm$2.10 & 36.47$\pm$0.13 \\
2 & 42.91$\pm$1.75 & 43.16$\pm$1.79 & 42.74$\pm$2.44 \\
3 & 43.76$\pm$0.65 & 44.35$\pm$2.25 & 46.43$\pm$1.07 \\
4 & 47.03$\pm$1.19 & 48.74$\pm$1.94 & 49.38$\pm$1.80 \\
5 & 49.16$\pm$1.77 & 50.20$\pm$1.25 & 51.61$\pm$1.09 \\
6 & 52.43$\pm$1.33 & 53.44$\pm$1.37 & 54.33$\pm$1.76 \\
7 & 52.81$\pm$1.55 & 53.89$\pm$0.78 & 54.59$\pm$1.04 \\
8 & 54.56$\pm$0.23 & 57.12$\pm$1.11 & 57.50$\pm$1.28 \\
9 & 58.08$\pm$1.48 & 59.62$\pm$1.50 & 60.35$\pm$1.80 \\
\bottomrule
\label{tab:retrain_cf10}
\end{tabular}
\end{sc}}
\end{table}

\begin{table}[H]
\centering
\small
\setlength{\tabcolsep}{4pt}{
\begin{sc}
\centering
\caption{\color{black} SVHN, VGG, Query Batch Size 500, Initial Set Size 500.}
\begin{tabular}{cccc}
\toprule
{\#Round} &{Random}&{Coreset} &{\dgal} \\
\midrule
0 & 52.68$\pm$1.97 & 52.74$\pm$6.16 & 52.59$\pm$3.73 \\
1 & 67.64$\pm$1.99 & 68.08$\pm$3.61 & 66.48$\pm$4.10 \\
2 & 73.46$\pm$1.51 & 74.93$\pm$1.44 & 74.34$\pm$2.22 \\
3 & 77.30$\pm$1.08 & 76.49$\pm$2.08 & 76.73$\pm$2.65 \\
4 & 79.27$\pm$0.78 & 79.33$\pm$0.72 & 80.19$\pm$0.78 \\
5 & 79.97$\pm$1.28 & 82.09$\pm$1.08 & 82.08$\pm$1.39 \\
6 & 83.97$\pm$0.42 & 82.30$\pm$0.33 & 83.80$\pm$1.30 \\
7 & 83.44$\pm$0.57 & 83.29$\pm$1.11 & 84.85$\pm$1.12 \\
8 & 86.24$\pm$0.52 & 84.72$\pm$0.52 & 86.59$\pm$1.25 \\
9 & 85.75$\pm$1.23 & 85.62$\pm$0.55 & 86.57$\pm$0.74 \\
\bottomrule
\label{tab:retrain_svhn}
\end{tabular}
\end{sc}}
\end{table}

\section{Discussion}

\paragraph{Limitation and Future Work.}
In the work, we study the connection between generalization performance and the training dynamics under the NTK regime. Although the relation between training dynamics and generalization performance has been verified by our experiments, the theoretical analysis of the relation out of the NTK regime still needs study. 
Besides, in the experiments, we mainly focus on the classification problem. Whether the proposed method is effective for the regression problem is under-explored. We would like to leave the study of the previously mentioned two problems in the future work.

\textcolor{black}{\paragraph{NTK Analysis for the Design of Practical Method.}
Although some works~\cite{lee2020finite, fort2020deep} discussed that the NTK assumption is hard to be strictly satisfied in some real-world models, we notice that some recent works have shown that the high-level conclusions derived based on NTK is insightful and useful for the design of practical models. Some of their applications can achieve SOTA. For example, Park et al.~\citep{park2020towards} used the NTK to predict the generalization performance of architectures in the application of Neural Architecture Search (NAS). Chen et al.~\citep{chen2021neural} used the condition number of NTK to predict a model’s trainability. Chen et al.~\citep{chen2021vision} also used the NTK to evaluate the trainability of several ImageNet models, such as ResNet. Deshpande et al.~\citep{deshpande2021linearized} used the NTK for model selection in the fine-tuning of pre-trained models on a target task. In our work, the  empirical results in Figure 3 and Appendix.E also show the effectiveness of the high-level conclusions derived from the theory still hold.
}

\paragraph{Social Impacts.}
In this work, we study the connection between the generalization performance and the training dynamics and try to bridge the gap between the theoretic findings of deep neural networks and deep active learning applications. We hope our work would inspire more attempts on the design of deep active learning algorithms with theoretical justification, which might have positive social impacts.
We do not foresee any form of negative social impact induced by our work.

\paragraph{License Privacy Information.}
We use the commonly used datasets, CIFAR10\footnote{\url{https://www.cs.toronto.edu/~kriz/cifar.html}}, SVHN\footnote{\url{http://ufldl.stanford.edu/housenumbers/}}, Caltech101\footnote{\url{https://data.caltech.edu/records/20086}} in the experiments. Those datasets follow the MIT, CC0 1.0, CC BY 4.0 License respectively and are publicly accessible.
No privacy information is included in those datasets.

\end{document}